\documentclass
[a4paper,11pt]{scrartcl}
\pdfoutput=1 
\usepackage[utf8]{inputenc}
\usepackage[T1]{fontenc}
\usepackage{microtype}
\usepackage[numbers]{natbib}

\title{Bounding Bloat in Genetic Programming}

\author{Benjamin Doerr$^\dagger$ \and Timo Kötzing$^\star$ \and J. A. Gregor Lagodzinski$^\star$ \and Johannes Lengler$^\diamond$}

\publishers{$\mbox{}$ \\ \small $\dagger:$ Laboratoire d’Informatique (LIX), École Polytechnique, Palaiseau, France \\
	$\star:$ Hasso Plattner Institute, University of Potsdam, Potsdam, Germany \\
	$\diamond:$ ETH Zürich, Zürich, Switzerland}

\date{\today}

\newcommand{\order}{\textsc{Order}\xspace}
\newcommand{\majority}{\textsc{Majority}\xspace}
\newcommand{\tinit}{T_{\mathrm{init}}}
\newcommand{\gptree}{GP-tree\xspace}
\newcommand{\gptrees}{GP-trees\xspace}
\newcommand{\Tmax}{T_{\mathrm{max}}}
\newcommand{\Tmin}{T_{\mathrm{min}}}

\newcommand{\var}{x}
\newcommand{\nonvar}{\overline{x}}
\newcommand{\Pois}{\mathrm{Pois}}
\newcommand{\EE}{\mathbb{E}}
\newcommand{\eps}{\varepsilon}

\newcommand{\BigO}{\mathrm{O}}
\newcommand{\LittleO}{\mathrm{o}}

\newcommand{\Prob}[1]{\mathrm{Pr}[#1 ]}
\newcommand{\Ex}[1]{\mathbb{E}[#1 ]}

\newcommand{\event}[1]{\ensuremath{\mathcal{{#1}}}}
\newcommand{\statespace}[1]{\ensuremath{{\mathscr{{#1}}}}}
\newcommand{\filtration}[1]{\ensuremath{{\mathscr{{#1}}}}}

\usepackage{xspace}
\usepackage{amsfonts,amssymb,amstext,amsthm,amsmath}
\usepackage{mathabx}
\usepackage[mathscr]{euscript}
\usepackage{bm}
\usepackage{mathtools}
\usepackage{xcolor}
\usepackage{caption}
\usepackage{subcaption}
\usepackage{calc}
\usepackage{booktabs}
\usepackage{graphicx}
\usepackage{enumerate}
\usepackage{enumitem}
\usepackage{lmodern}
\usepackage{datetime}
\usepackage{multirow}
\usepackage{tabularx}

\usepackage[english]{babel}

\usepackage[ruled,vlined,linesnumbered,algo2e]{algorithm2e}
\SetKwFor{Function}{function}{is}{end function}
\newcommand{\assign}{\leftarrow}
\SetKw{Input}{input}
\SetKw{Output}{output}
\SetKw{Initialization}{initialization}

\theoremstyle{plain}
\newtheorem{theorem}{Theorem}[section]
\newtheorem{corollary}[theorem]{Corollary}
\newtheorem{lemma}[theorem]{Lemma}

\theoremstyle{definition}

\newcommand{\ignore}[1]{}

\newcommand{\oneonegp}{(1+1)~GP\xspace}

\newcommand{\set}[2]{\{#1 \; | \; #2\}}

\newcommand{\N}{\mathbb{N}}

\begin{document}

\maketitle

\begin{abstract}
While many optimization problems work with a fixed number of decision variables and thus a fixed-length representation of possible solutions, genetic programming (GP) works on variable-length representations. A naturally occurring problem is that of bloat (unnecessary growth of solutions) slowing down optimization. Theoretical analyses could so far not bound bloat and required explicit assumptions on the magnitude of bloat.

In this paper we analyze bloat in mutation-based genetic programming for the two test functions \order and \majority. We overcome previous assumptions on the magnitude of bloat and give matching or close-to-matching upper and lower bounds for the expected optimization time.

In particular, we show that the \oneonegp takes (i) $\Theta(\tinit + n \log n)$ iterations with bloat control on \order as well as \majority; and (ii) $\BigO(\tinit \log \tinit + n (\log n)^3)$ and $\Omega(\tinit + n \log n)$ (and $\Omega(\tinit \log \tinit)$ for $n=1$) iterations without bloat control on \majority.\footnote{An extended abstract of the paper at hand has been published at GECCO 2017}
\end{abstract}

\section{Introduction}
\label{sec:intro}
While much work on nature-inspired search heuristics focuses on representing problems with strings of a fixed length (simulating a genome), genetic programming considers trees of variable size. One of the main problems when dealing with a variable-size representation is the problem of \emph{bloat}, meaning an unnecessary growth of representations, exhibiting many redundant parts and slowing down the search.

In this paper we study the problem of bloat from the perspective of run time analysis. We want to know how optimization proceeds when there is no explicit bloat control, which is a setting notoriously difficult to analyze formally: Previous works were only able to give results conditional on strong assumptions on the bloat (such as upper bounds on the total bloat), see~\cite{NguUrlWag:c:13:Foga} for an overview. The only exception is the very recent work~\cite{koetzing2018crossover} continuing the line of research presented here.

We use recent advances from drift theory as well as other tools from the analysis of random walks to bound the behavior and impact of bloat, thus obtaining unconditional bounds on the expected optimization time even when no bloat control is active.


\begin{table*}[t]
	\caption{\textbf{Summary of best known bounds.} Note that $T_{\mathrm{max}}$ denotes the maximal size of the best-so-far tree in the run until optimization finished (we consider bounds involving $T_{\mathrm{max}}$ as conditional bounds).
	}\label{table:overviewResults}
	{\renewcommand{\arraystretch}{1.4}
	\begin{tabularx}{\textwidth}{@{}l c X X@{}}
		\toprule
		\textbf{Problem} & \textbf{$k$} & \textbf{Without Bloat Control} & \textbf{With Bloat Control}\\
		\midrule
		\multirow{3}{*}{\order}
			& $1$ & $\BigO(nT_{\mathrm{max}})$, \cite{GPFOGA11} 
		& $\Theta(\tinit  + n \log n)$, \cite{NeuGECCO12} \\
			\cmidrule{2-4}
		& \multirow{2}{*}{$1 + \Pois(1)$}
		& \multirow{2}{*}{$\BigO(nT_{\mathrm{max}})$, \cite{GPFOGA11}}
		& $\Theta(\tinit + n \log n)$, \\
		& & & Theorem~\ref{thm:bloatControl}\\
		\midrule[\cmidrulewidth]
		\multirow{12}{*}{\majority} 
		& \multirow{6}{*}{$1$}
		& $\BigO(\tinit \log \tinit + n \log^3 n)$, Theorem~\ref{thm:majority} 
		& \multirow{6}{*}{$\Theta(\tinit  + n \log n)$, \cite{NeuGECCO12}}\\
		&	& $\Omega(\tinit \log \tinit)$, $n=1$,\\
		&	&  Theorem~\ref{thm:majority_lbound} \\
		&	& $\Omega(\tinit  + n \log n)$, \\
		&	& Theorem~\ref{thm:majority_lbound} \\
		\cmidrule{2-4}
		& \multirow{6}{*}{$1 + \Pois(1)$}
		& $\BigO(\tinit \log \tinit + n \log^3 n)$, Theorem~\ref{thm:majority} 
		& \multirow{6}{*}{\shortstack[l]{$\Theta(\tinit  + n \log n)$,\\ Theorem~\ref{thm:bloatControl}}}\\
		&	& $\Omega(\tinit \log \tinit)$, $n=1$,\\
		&	&  Theorem~\ref{thm:majority_lbound}\\
		&	& $\Omega(\tinit  + n \log n)$,\\
		&	& Theorem~\ref{thm:majority_lbound}\\
		\bottomrule
	\end{tabularx}
}
\end{table*}

Our focus is on mutation-based genetic programming (GP) algorithms, which has been a fruitful area for deriving run time results in GP. We will  be concerned with the problems \order and \majority as introduced in \cite{GoldbergO98}. This is in contrast to other theoretical work on GP algorithms which considered the PAC learning framework~\cite{DBLP:conf/gecco/KotzingNS11} or the Max-Problem~\cite{KoeSutNeuOre:c:12} as well as Boolean functions~\cite{MoraglioMM13,MambriniM14,MambriniO16}.

Individuals for \order and \majority are binary trees, where each inner node is labeled $J$ (short for \emph{join}, but without any associated semantics) and leaves are labeled with literal symbols; we call such trees \emph{\gptrees}. The set of literal symbols is $\set{\var_i}{i \leq n} \cup \set{\nonvar_i}{i \leq n}$, where $n$ is the number of variables. In particular, literal symbols are paired ($\var_i$ is paired with $\nonvar_i$). We say that in a \gptree $t$ a leaf $u$ comes \emph{before} a leaf $v$ if $u$ comes before $v$ in an in-order parse of the tree.

For the \order problem fitness is assigned to \gptrees as follows: we call a variable $i$ \emph{expressed} if there is a leaf labeled $\var_i$ and all leaves labeled $\nonvar_i$ do not come before that leaf. The fitness of a \gptree is the number of its expressed variables $i$.

For the \majority problem, fitness is assigned to \gptrees as follows. We call a variable $i$ \emph{expressed} if there is a leaf labeled $\var_i$ and there are at least as many leaves labeled $\var_i$ as there are leaves labeled $\nonvar_i$ (the positive instances are in the majority). Again, the fitness of a \gptree is the number of its expressed variables $i$.

A first run time analysis of genetic programming on \order and \majority was conducted in \cite{GPFOGA11}. This work considered the algorithm \oneonegp proceeding as follows. A single operation on a \gptree~$t$ chooses a leaf $u$ of $t$ uniformly at random and randomly either relabels this leaf (to a random literal symbol), deletes it (i.e. replacing the parent of $u$ with the sibling of $u$) or inserts a leaf here (i.e., replaces $u$ with an inner node with one randomly labeled child and $u$ as the other child, in random order). The \oneonegp is provided with a parameter $k$ which determines how many such operations make up an atomic mutation; in the simplest case with $k=1$, but a random choice of $k=1+\mathrm{Pois}(1)$ (where $\mathrm{Pois}(1)$ denotes the Poisson distribution with parameter $\lambda=1$) is also frequently considered. The \oneonegp then proceeds in generations with a simple mutation/selection scheme (see Algorithm~\ref{alg:gp}).

A straightforward version of bloat control for this algorithm was introduced in \cite{Luke:2002:GECCO} as \emph{lexicographic parsimony pressure}. Here the algorithm always prefers the smaller of two trees, given equal fitness. For this \cite{NeuGECCO12} was able to give tight bounds on the optimization time in the case of $k=1$: in this setting no new redundant leaves can be introduced. The hard part is now to give an analysis when $k = 1 + \mathrm{Pois}(1)$, where bloat can be reintroduced whenever a fitness improvement is achieved (without fitness improvements, only smaller trees are acceptable). With a careful drift analysis, we show that in this case we get an (expected) optimization time of $\Theta(\tinit + n \log n)$ (see Theorem~\ref{thm:bloatControl}). Previously, no bound was known for \majority and the bound of $\BigO(n^2 \log n)$ for \order required a condition on the initialization.

Without such bloat control it is much harder to derive definite bounds. From~\cite{GPFOGA11} we have the conditional bounds of $\BigO(nT_{\mathrm{max}})$ for \order using either $k=1$ or $k=1+\mathrm{Pois}(1)$, where $T_{\mathrm{max}}$ is an upper bound on the maximal size of the best-so-far tree in the run (thus, these bounds are conditional on these maxima not being surpassed). For \majority and $k=1$ \cite{GPFOGA11} gives the conditional bound of $\BigO(n^2 T_{\mathrm{max}} \log n)$. We focus on improving the bound for \majority and obtain a bound of $\BigO(\tinit \log \tinit + n \log^3 n)$ for both $k=1$ and $k=1+\mathrm{Pois}(1)$ (see Theorem~\ref{thm:majority}). The proof of this theorem requires significant machinery for bounding the extent of bloat during the run of the optimization.

The paper is structured as follows. In Section~\ref{sec:preliminaries} we will give a short introduction to the studied algorithm. In Section~\ref{sec:driftTheorems} the main tool for the analysis is explained, that is the analysis of drift. Here we state a selection of known theorems as well as a new one (Theorem~\ref{thm:Multi_Drift_Bounded_Step_Size}), which gives a lower bound conditional on a multiplicative drift with a bounded step size. 
In Section~\ref{sec:bloatControl} we will study the case of bloat control given $k=1+\Pois(1)$ operations in each step. Subsequently we will study \majority without bloat control in Section~\ref{sec:majority}. Section~\ref{sec:conclusion} concludes this paper.

\section{Preliminaries}
\label{sec:preliminaries}
In this section we make the notions introduced in Section~\ref{sec:intro} more formal. We consider tree-based genetic programming, where a possible solution to a given problem is given by a syntax tree. The inner nodes of such a tree are labeled by function symbols from a set $F_S$ and the leaves of the tree are labeled by terminals from a set $T$.

We analyze the problems \order and \majority, whose only function is the join operator (denoted by $J$). The terminal set $X$ consists of $2n$ literals, where $\nonvar_i$ is the complement of $\var_i$:
\begin{itemize}
	\item $F_S \coloneqq \{J\}$, $J$ has arity $2$,
	\item $X\coloneqq \{\var_1, \nonvar_1, \dots , \var_n, \nonvar_n \}$.
\end{itemize}

For a given syntax tree $t$, the value of the tree is computed by parsing the tree in-order and generating the set $S$ of \emph{expressed} variables in this way. For \order a variable $i$ is expressed if a literal $\var_i$ is present in $t$ and there is no $\nonvar_i$ that is visited in the in-order parse before the first occurrence of $\var_i$. For \majority a variable $i$ is expressed if a literal $\var_i$ is present in $t$ and the number of literals $\var_i$ is at least the number of literals $\nonvar_i$.

In this paper we consider simple mutation-based genetic programming algorithms which use a modified version of the \emph{Hierarchical Variable Length} (HVL) operator (\cite{OReilly:thesis}, \cite{OReilly:1994:GPSAHC}) called \emph{HVL-Prime} as discussed in~\cite{GPFOGA11}. HVL-Prime allows to produce trees of variable length by applying three different operations: insert, delete and substitute (see Figure~\ref{fig:HVL}). Each application of HVL-Prime chooses one of these three operations uniformly at random, where $k$ denotes the number of applications of HVL-Prime we allow for each mutation.
\begin{figure*}[t]
	{\renewcommand{\arraystretch}{1.5}
		\begin{tabularx}{\textwidth}{|l X|}
			\hline 
			\multicolumn{2}{|>{\hsize=\dimexpr1\hsize+10\tabcolsep}X|}{\centering Given a GP-tree $t$, mutate $t$ by applying HVL-Prime. For each application, choose uniformly at random one of the following three options.}\\ 
			substitute	& Choose a leaf uniformly at random and substitute it with a leaf in $X$ selected uniformly at random. \\ 
			insert  & Choose a node $v \in X$ and a leaf $u \in t$ uniformly at random. Substitute $u$ with a join node $J$, whose children are $u$ and $v$, with the order of the children chosen uniformly at random. \\ 
			delete & Choose a leaf $u \in t$ uniformly at random. Let $v$ be the sibling of $u$. Delete $u$ and $v$ and substitute their parent $J$ by $v$.  \\ 
			\hline 
		\end{tabularx}
	}
	\caption{Mutation operator HVL-Prime.}
	\label{fig:HVL}
\end{figure*}
We associate with each tree $t$ the complexity $C$, which denotes the number of nodes $t$ contains. Given a function $F$, we aim to generate an instance $t$ maximizing $F$.

\begin{figure}[ht]
	\centering
	\includegraphics[width=.8\textwidth]{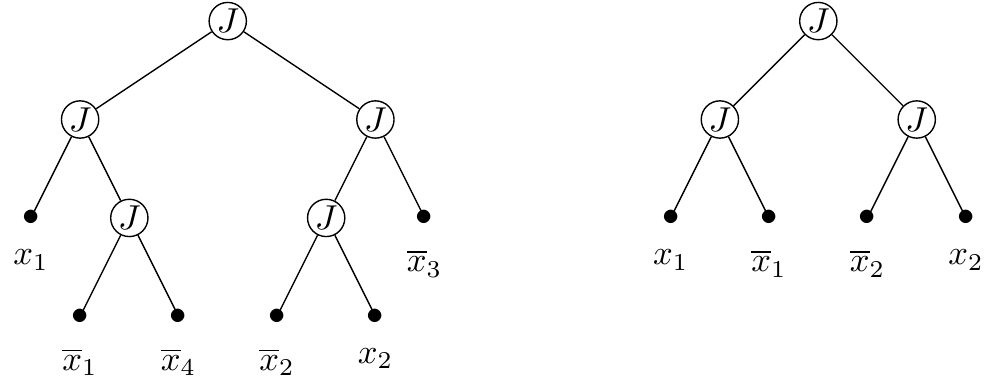}
	\caption{Two GP-trees with the same fitness. For \order the fitness is $1$ since only the first variable occurs with a non-negated literal first. For \majority the fitness is $2$, since the variable $1$ and $2$ have one literal $\var_i$ and also one literal $\nonvar_i$. However, the left one has complexity $11$ whereas the other has complexity $7$.}
	\label{fig:tree_examples}
\end{figure}
We consider two problems. The first one is the single problem of computing a tree $t$ which maximizes $F$. During an optimization run we can use the complexity $C$ to generate an order for solutions with the same fitness by preferring solutions with smaller complexity (see Figure~\ref{fig:tree_examples}). This gives us a way of breaking ties between solutions with the same fitness. Hence, the second problem consists of maximizing the multi-objective function given by $F$ and $C$.

Consequently, we study the following problems:
\begin{itemize}
	\item \order and \majority without bloat control, which consist of maximizing the given function without studying the complexity.
	\item \order and \majority with bloat control, which consist of maximizing the given function and preferring solutions with smaller complexity, if two solutions have the same function value.
\end{itemize}

In order to solve these problems we study the \oneonegp proceeding as follows. It starts with a given initial tree with $\tinit$ leaves and tries to improve its fitness iteratively. In each iteration, the number of mutation steps $k$ is chosen according to a fixed distribution; important options for this distribution is (i) constantly $1$ and (ii) $1 + \Pois(1)$, where $\Pois(\lambda)$ denotes the Poisson distribution with parameter $\lambda$. The choices for $k$ in the different iterations are independent. The \oneonegp then produces an offspring from the best-so-far individual by applying mutation $k$ times in a row; the offspring is discarded if its fitness is worse than the best-so-far, otherwise it is kept to replace the previous best-so-far. Recall that the fitness in the case with bloat control contains the complexity as a second order term. Algorithm~\ref{alg:gp} states the \oneonegp more formally.

\begin{algorithm2e}[H]
	\caption{\oneonegp}
	\label{alg:gp}
	Let $t$ be the initial tree\;
	\While{optimum not reached}{%
		$t' \assign t$\;
		Choose $k$\;
		\For{$i = 1$ \KwTo $k$}{
			$t' \assign \mbox{mutate}(t')$\;
		}
		\lIf{$f(t') \geq f(t)$}{$t \assign t'$}
	}
\end{algorithm2e}

\section{Drift Theorems and Preliminaries}
\label{sec:driftTheorems}
In this section we collect theorems on stochastic processes that we will use in the proofs. We apply the standard Landau notation $\BigO(\cdot)$, $\LittleO(\cdot)$, $\Omega(\cdot)$, $\omega(\cdot)$, $\Theta(\cdot)$ as detailed in~\cite{CormenAlgorithms}. 

\begin{theorem}[Chernoff Bound~\cite{dubhashi2009concentration}]\label{thm:Chernoff}
	Let $X_1, \ldots, X_n$ be independent random variables that take values in $[0,1]$. Let $X= \sum_{i=1}^n X_i$ and $\mu = \Ex{X}$. Then for all $0 \leq \delta \leq 1$,
	\[
	\Pr[X \leq (1-\delta)\mu] \leq e^{-\delta^2\mu/2}
	\]
	and
	\[
	\Pr[X \geq (1+\delta)\mu] \leq e^{-\delta^2\mu/3}.
	\]
\end{theorem}


We will apply a variety of drift theorems to derive the results of this paper. \emph{Drift}, in this context, describes the \emph{expected change} of the best-so-far solution within one iteration with respect to some \emph{potential}. In later proofs we will define potential functions on best-so-far solutions and prove bounds on the drift; these bounds then translate to expected run times with the use of the drift theorems from this section. We use formulations from~\cite{lengler2018drift} because they do not require finite search spaces, and they do not require that the potential forms a Markov chain. Instead, we will have random variables $Z_t$ (the current \gptree) that follow a Markov chain, and the potential is some function of $Z_t$. We start with a theorem for \emph{additive drift}.
\begin{theorem}[Additive Drift~\cite{HeYao:04:drift}, formulation of~\cite{lengler2018drift}]
\label{thm:additiveDrift}
Let $(Z_t)_{t \in\N_0}$ be random variables describing a Markov process with state space $\statespace{Z}$, and with a potential function $\alpha: \statespace{Z} \to {S} \subseteq [0,\infty)$, and assume $\alpha(Z_0)=s_0$. Let $T\coloneqq\inf\{t \in \N_0 \mid \alpha(Z_t)=0\}$ be the random variable that denotes the earliest point in time $t \geq 0$ such that $\alpha(Z_t) = 0$. If there exists $c > 0$ such that for all $z\in \statespace{Z}$ with $\alpha(z) > 0$ and for all $t\geq 0$ we have
\begin{equation}\label{eq:additivedrift}
\EE[\alpha(Z_{t+1})\mid Z_t=z]\le \alpha(z) - c,
\end{equation}
then
\[
\EE[T] \leq \frac{s_0}{c}.
\]
\end{theorem}

We will use the following \emph{variable drift theorem}, an extension of the variable drift theorem from \cite[Theorem~4.6]{Joh:th:10}.
\begin{theorem}[Variable Drift~\cite{lengler2018drift}]\label{thm:variableDrift}
Let $(Z_t)_{t \in\N_0}$ be a Markov chain with state space $\statespace{Z}$ and with a potential function $\alpha: \statespace{Z} \to {S} \subseteq \{0\} \cup [s_{\min},\infty)$ for some $s_{\min} >0$. Assume $\alpha(Z_0)=s_0$, and let $T \coloneqq\inf\{t \in \N_0 \mid \alpha(Z_t)=0\}$ be the random variable that denotes the first point in time $t \in \mathbb{N}$ for which $X_t = 0$. Suppose furthermore that there exists a positive, increasing function $h: [s_{\min},\infty) \rightarrow \mathbb{R}^+$ such that for all $z \in \statespace{Z}$ with $\alpha(z)> 0$ and all $t\geq 0$ we have
\[
\EE[\alpha(Z_{t+1})\mid Z_t=z]\le \alpha(z)-h(\alpha(z)).
\]
Then
\[
\EE[T] \leq \frac{1}{h(1)}+\int_{1}^{s_0} \frac{1}{h(u)}du.
\]
\end{theorem}

The most important special case is for \emph{multiplicative drift}, which was developed in~\cite{doerr2013adaptive}. We again give the version from~\cite{lengler2018drift} 
\begin{theorem}[Multiplicative Drift~\cite{lengler2018drift}]\label{thm:multiplicativeDrift}
Let $(Z_t)_{t \in\N_0}$ be a Markov chain with state space $\statespace{Z}$ and with a potential function $\alpha: \statespace{Z} \to {S} \subseteq \{0\} \cup [s_{\min},\infty)$ for some $s_{\min} >0$, and assume $\alpha(Z_0)=n$. Let $T \coloneq \inf\{t \in \N_0:\alpha(Z_t)=0\}$ be the random variable that denotes the first point in time $t \in \mathbb{N}$ for which $X_t = 0$. Assume that there is $\delta>0$ such that for all $z \in \statespace{Z}$ with $\alpha(z)> 0$ and for all $t\geq 0$ we have
\[
\EE[\alpha(Z_{t+1})\mid Z_t=z]\le  (1-\delta)\alpha(z).
\]
Then for all $k >0$
\[
\Pr\left[T > \left\lceil \frac{\log(n/s_{\min}) +k}{\delta} \right\rceil\right] \leq e^{-k},
\]
and
\[
\EE\left[T\right] \leq \frac{1+\log(n/s_{\min})}{\delta}.
\]

\end{theorem}

For bloat estimation we need a lower bound drift theorem in the regime of weak additive drift. A related theorem (Theorem~\ref{driftthm:additivetail}) follows from Theorem 10 and 12 in~\cite{kotzing2016concentration}. Theorem~\ref{driftthm:additivetail} is not directly applicable to our situation, since it gives only tight bounds in the regime of strong drift. Nevertheless, we can use it to prove lower bounds on the tail probabilities for the regime of weak drift, see Theorem~\ref{lem:weakdrift} below. 
\begin{theorem}[{Strong Additive Drift, Lower Tail Bound, follows from~\cite[Theorem 10,12]{kotzing2016concentration}}]
\label{driftthm:additivetail}
Let $(Z_t)_{t \in\N_0}$ be random variables describing a Markov process with state space $\statespace{Z}$, and with a potential function $\alpha: \statespace{Z} \to {S} \subseteq \N$, and assume $\alpha(Z_0)=s_0$. Suppose further that there exist $\delta, \rho, r >0$ such that for all $z \in \statespace{Z}$ such that $\alpha(z) > 0$, all $k\in \N_0$, and all $t \geq 0$,
\begin{enumerate}
\item $\Pr[|X_{t} - X_{t+1}| > k\mid Z_t=z] \leq \frac{r}{(1+\delta)^{k}}.$
\item $\EE[X_t - X_{t+1} \mid Z_t = z] \leq \rho.$
\end{enumerate}
Then, for all $x \geq 0$, if $T\coloneqq\inf\{t \in \N_0 \mid \alpha(Z_t)=0\}$ is the random variable that denotes the earliest point in time $t \geq 0$ such that $\alpha(Z_t) = 0$.
\begin{align}\label{drifteq:addtail3}
\Pr\left[T \leq \frac{s_0-x}{\rho}\right] \leq \exp\left\{-\frac{\delta x}{8} \, \min\Big\{1, \frac{\delta^2\rho x}{32rs_0}\Big\}\right\} .
\end{align}
\end{theorem}

The next theorem gives a lower bound on hitting times of random walks even if we start close to the goal, provided that the drift towards the goal is weak. We remark that the statement on the expectation is similar to other lower bounds for additive drift~\cite{kotzing2016concentration}, but the existing tail bounds are tailored to the regime of strong drift, and are thus not tight in our case. We prove it by martingale theory.

\begin{theorem}[Weak Additive Drift, Lower Bounds]\label{lem:weakdrift}
For every $\delta,C >0$ there exists $\eps >0$ such that the following holds for all $N\geq 0$. Let $(Z_t)_{t \in\N_0}$ be random variables describing a Markov process with state space $\statespace{Z}$, and with a potential function $\alpha:  {\statespace{Z}} \to S \subseteq [0,\infty)$. We denote $X_t \coloneqq \alpha(Z_t)$. Assume $\alpha(Z_0)=s_0$ and that the following conditions hold for all $t\geq 0$ and all $z,z' \in \statespace{Z}$ such that $\alpha(z) \leq N$.
\begin{enumerate}[label=(\roman*)]
\item\label{it:drift} \emph{Weak Drift.} $\EE[X_{t}- X_{t+1} \mid Z_t =z] \leq C/N$.
\item\label{it:steps} \emph{Small Steps.} $\Pr[|X_{t} -X_{t+1}| \geq k \mid Z_t = z'] \leq (1+\delta)^{-k}$.
\item\label{it:increase} \emph{Initial Increase.} $\Pr[X_{t+1} > X_t \mid Z_t=z] \geq \delta$.
\end{enumerate}
Then for every $0\leq x < s_0 \leq \eps N$, if $T \coloneqq \min\{\tau \geq 0 \mid X_t \leq x\}$ is the hitting time of $\{0,1,\ldots,x\}$ for $X_t$, then
\begin{align}\label{eq:jhexpectation}
\EE[T] \geq \eps (s_0-x)N
\end{align}
and
\begin{align}\label{eq:jhtailbound}
\Pr[T \geq \eps N^2] \geq \frac{\eps}{N}.
\end{align}
\end{theorem}

\begin{proof}
Note that for any constant $N_0 = N_0(\delta,C)$, the statement is trivial for all $N \leq N_0$ if $\eps$ is sufficiently small. Hence, we may always assume that $N$ is large compared to $\delta$ and $C$. 

Without loss of generality, we may assume that $|\EE[X_{t+1}- X_t \mid Z_t = z]| \le C/N$, which stronger than~\ref{it:drift}. If this does not hold a priori, then we may couple the process $X_t$ to a process $X_{t}'$ which makes the same step as $X_t$ (i.e., $X_{t+1}-X_t = X_{t+1}'-X_t'$), with one exception: if $\EE[X_{t}- X_{t+1} \mid Z_t = z] < -C/N$ at any point in time, then with some (additional) probability $p_t$ we choose $Z_{t+1}$ such that $X_{t+1}$ is smaller, thus increasing the drift. More precisely, we choose $p_t$ in such a way that $-C/N \leq \EE[X_{t}- X_{t+1} \mid Z_t = z] \leq C/N$. Then $X_t' \leq X_t$ for all $t\geq 0$, so it suffices to prove the statement for $X_t'$. To keep notation simple, we will assume that we do not need to modify $X_t$ in the remainder.

We rescale $\tilde X_t \coloneqq X_t-x$ and consider the drift of $\tilde X_t^2$. 
Let $p_i \coloneqq \Pr[X_{t+1}-X_t = i \mid Z_t = z]$ for all $i\in \mathbb Z$. Then  
\begin{align*}
\EE[\tilde X_{t+1}^2-\tilde X_t^2 \mid Z_t = z] & = \sum_{i\in \mathbb Z} p_i (\tilde X_t+i)^2-\tilde X_t^2 = \sum_{i\in \mathbb Z} p_i (2\tilde X_t i + i^2) \\ 
&= 2 \tilde X_t \EE[X_{t+1}-X_t \mid Z_t = z] + \sum_{i\in \mathbb Z} p_i i^2. 
\end{align*}
Note that we have $\sum_{i\in \mathbb Z} p_i i^2 \ge p_1 \ge \delta$ by~\ref{it:drift} and 
$\sum_{i\in \mathbb Z} p_i i^2 \le \sum_{i\in \mathbb Z} (1+\delta)^{|i|} i^2 \in \BigO(1)$ by~\ref{it:steps}.
Together with Condition \ref{it:drift}, we have for all $0 \leq \tilde X_t \leq \delta N/(4C)$,  
\begin{align}
\delta/2 \le \EE[\tilde X_{t+1}^2-\tilde X_t^2 \mid Z_t = z] \le \BigO(1). 
\label{eq:squaredrift}
\end{align}

Let $t_0$ be the (random) time when the process $\tilde X_t$ (started at $\tilde X_0 = s_0-x$)  for the first time leaves the interval $I=[1,\delta N/(4C)-x]$ on either side. We note that $t_0 \le T$ holds. Let $p_\ell$ and $p_r$ be the probabilities that the process leaves the interval on the left (that is, at $0$ or lower) and on the right (that is, at $\lfloor\delta N/(4C)-x\rfloor+1$ or higher), respectively. By~\ref{it:steps} if the process leaves $I$ on the right side, then the expectation of $\tilde X_t$ in this case is at most $\delta N/(2C)$; recall that we assumed $N$ to be large. Similarly, if it leaves $I$ on the left, then the expectation of $\tilde X_t$ is at least $-D$ for some constant $D>0$.

By~\eqref{eq:squaredrift} there is a constant $D>0$ such that the process $Y_t \coloneqq \tilde X_t^2 - Dt$ has a negative drift in the interval $I$. Hence, using that $t_0$ is a stopping time we obtain from the optional stopping theorem~\cite{grimmett2001probability}
\begin{align}\label{eq:jh1}
(s_0-x)^2 = \EE[Y_0] \ge \EE[Y_{t_0}] & \ge p_r  \left(\frac{\delta N}{4C}-x\right)^2 - p_\ell D - D \EE[t_0] \nonumber \\
&\ge p_r  \left(\frac{\delta N}{8C}\right)^2 - D- D \EE[t_0] .
\end{align}
Similarly, we regard the process $U_t = \tilde X_t + Ct/N$. By~\ref{it:drift} it has a non-negative drift for $t < t_0$. Hence, we obtain
\begin{equation}
  s_0-x = \EE[U_0] \le \EE[U_{t_0}] \le p_r  \frac{\delta N}{2C} + \frac{C \EE[t_0]}{N}.
\label{eq:jh2}
\end{equation}
This yields a lower bound of $p_r \delta N^2/(2C) \ge (s_0-x)N - C \EE[t_0]$ for $p_r$. Together with~\eqref{eq:jh1} we obtain
\begin{equation}
  \EE[t_0]
  \ge \frac{(s_0-x)\left(\delta N/(2C) - 16(s_0-x)\right)-16D}{16D+\delta/2},
\label{eq:jhresult}
\end{equation}
which proves the bound on the expectation~\eqref{eq:jhexpectation} since $s_0-x \leq \eps N$.

For the tail bound~\eqref{eq:jhtailbound} we reverse the previous argument. By~\eqref{eq:squaredrift} the process $U_t \coloneqq \tilde X_t^2 - \delta t/2$ has a non-negative drift in the interval $I$. If $\tilde X_t$ leaves $I$ on the right side then due to~\ref{it:steps} the expectation of $\tilde X_t^2$ is at most $(\delta N/(2C))^2$. Hence, by the optional stopping theorem
\begin{align*}
(s_0-x)^2 = \EE[U_0] \le \EE[U_{t_0}] & \le p_r  \left(\frac{\delta N}{2C}\right)^2 - \frac{\delta}{2} \EE[t_0] \stackrel{\eqref{eq:jhexpectation}}{\le} p_r\left(\frac{\delta N}{2C}\right)^2 - \frac{\delta}{2}\eps (s_0-x)N.
\end{align*}
Solving for $p_r$ shows that $p_r \in \Omega(1/N)$ whenever $s_0-x \leq \delta\eps/4 \, N$. Note that we may assume the latter condition by decreasing the $\eps$ in the theorem. (Despite the formulation, it is obviously sufficient to prove~\eqref{eq:jhexpectation} for $\eps$ and~\eqref{eq:jhtailbound} for $\eps' \coloneqq \delta\eps/4$.) Then with probability $\Omega(1/N)$ we have $X_t > \delta N/(4C)$ for some $t\geq 0$. However, starting from this $X_t$ by Theorem~\ref{driftthm:additivetail} with probability $\Omega(1)$ we need at least $\Omega(N^2)$ additional steps to return to $x < \eps N$ if $\eps < \delta/(4C)$. This proves~\eqref{eq:jhtailbound}.
\end{proof}

For our lower bounds we need the following new drift theorem, which allows for non-monotone processes (in contrast to, for example, the lower bounding multiplicative drift theorem from \cite{witt:2013:cpc}), but requires an absolute bound on the step size.

\begin{theorem}[Multiplicative Drift, lower bound, bounded step size]\label{thm:Multi_Drift_Bounded_Step_Size}
	Let $(Z_t)_{t \in\N_0}$ be random variables describing a Markov process with state space $\statespace{Z}$ with a potential function $\alpha: \statespace{Z} \to {S} \subseteq (0,\infty)$, for which we assume $\alpha(Z_0)=s_0$. Let $\kappa > 0$, $s_{\mathrm{min}} \geq \sqrt{2}\kappa$ and let $T \coloneqq \inf\{t \in \N_0 \mid \alpha(Z_t) \le s_{\mathrm{min}}\}$ be the random variable denoting the earliest point in time $t \geq 0$ such that $\alpha(Z_t) \leq s_{\mathrm{min}}$. If there exists a positive real $\delta > 0$ such that for all $z \in \statespace{Z}$ with $\alpha(z) > s_{\mathrm{min}}$ and all $t \geq 0$ it holds
	\begin{enumerate}
		\item $|\alpha(Z_t) - \alpha(Z_{t+1})| \leq \kappa$ , and
		\item $\Ex{\alpha(Z_t) - \alpha(Z_{t+1}) \mid Z_t = z} \leq \delta \alpha(z)$,
	\end{enumerate}
	then
	\[
	\Ex{T} \geq \frac{1+ \ln(s_0) - \ln(s_{\mathrm{min}})}{2\delta + \frac{\kappa^2}{s_{\mathrm{min}}^2 - \kappa^2}}.
	\]
\end{theorem}
\begin{proof}
	We concatenate $\alpha$ with a second potential function $g$ turning the multiplicative bound of the expected drift into an additive bound enabling us to apply the additive drift theorem.
	Let 
	\[
	g(s) \coloneqq 1 + \ln \left( \frac{s}{s_{\mathrm{min}}}\right)
	\]
	and $g(0) \coloneqq 0$. Furthermore, let $X_t \coloneqq \alpha(Z_t)$ and $V_t \coloneqq g(X_t) = g(\alpha(Z_t))$. It follows that $V_t$ is a stochastic process over the search space $R=g(\alpha(\mathrm Z)) \cup \{ 0\}$. We observe that $T$ is also the first point in time $t \in \mathbb{N}$ such that $V_t \leq 1$. Since $s_{\mathrm{min}}$ is a lower bound on $X_t$, $s_{\mathrm{min}} - \kappa$ is a lower bound on $X_{t+1}$. Thus, $X_{t+1}>0$ as well as $V_{t+1} >0$. We derive
	\[
	V_t - V_{t+1} = \ln \left( \frac{X_t}{X_{t+1}}\right).
	\]
	Therefore, due to Jensen's inequality we obtain
	\[
	\Ex{V_t - V_{t+1} \mid Z_t = z} \leq \ln \left( \EE\left[\frac{X_t}{X_{t+1}} \,\middle|\, Z_t = z\right] \right).
	\]
	The value of $X_{t+1}$ can only be in a $\kappa$-interval around $X_{t}$ due to the bounded step size. For all $i\geq 0$ let $p_i$ be the probability that $X_t - X_{t+1} = i$ and let $q_i$ be the probability that $X_t - X_{t+1} = -i$. Let $z \in \statespace{Z}$ and $s \coloneqq \alpha(z)$. We note that $p_0 = q_0$ and obtain by counting twice the instance of a step size of $0$ 
	\begin{align*}
	\EE\left[ \frac{X_t}{X_{t+1}} \,\middle|\, Z_t = z\right] &\leq \left( \sum_{i=0}^{\kappa} \frac{s}{s-i} \, p_i + \frac{s}{s+i} \, q_i\right) = \left( \sum_{i=0}^{\kappa} s \,\frac{p_i (s+i ) + q_i (s-i)}{s^2-i^2}\right) \\
	&\leq \left( \sum_{i=0}^{\kappa} s \,\frac{p_i (s+i ) + q_i (s-i)}{s^2-\kappa^2}\right) = \left( \frac{s^2}{s^2 - \kappa^2} + \sum_{i=0}^{\kappa} \frac{s  (ip_i - iq_i)}{s^2-\kappa^2}\right),
	\end{align*}
	where the last equality comes from summing all non-zero probabilities for a step size, i.e. $\sum p_i + q_i = 1$. The same holds for $X_t$ since $s_{\mathrm{min}}\geq \sqrt{2} \kappa$. It follows that $X_t^2 - \kappa^2 \geq 1/2 X_t^2$ and this yields
	\begin{align*}
	\EE\left[ \frac{X_t}{X_{t+1}} \,\middle|\, Z_t = z\right] &\leq \left( \frac{s^2}{s^2 - \kappa^2}+ \frac{2}{s} \sum_{i=0}^{\kappa} i p_i - iq_i \right)= \left( 1 + \frac{\kappa^2}{s^2 - \kappa^2} + \frac{2}{s} \sum_{i=0}^{\kappa} i p_i - i q_i \right).
	\end{align*}
	Since the remaining sum in the log-term is the difference of $X_t$ and $X_{t+1}$ multiplied by the probability for the step size, we obtain
	\begin{align*}
	\Ex{V_t - V_{t+1} \mid X_t = s} &\leq  \ln \left( 1 + \frac{\kappa^2}{X_t^2 - \kappa^2} +  2 \EE \left[\frac{X_t - X_{t+1}}{X_t} \, \middle| \, Z_t = z \right] \right) \\
	&\leq  2 \EE \left[\frac{X_t - X_{t+1}}{X_t} \, \middle| \, Z_t = z \right] + \frac{\kappa^2}{X_t^2 - \kappa^2} \leq  2 \delta + \frac{\kappa^2}{X_t^2 - \kappa^2} .
	\end{align*}
	Finally, we apply the additive drift theorem and deduce
	\begin{align*}
	\Ex{T} & \geq \frac{V_0}{2 \delta + \frac{\kappa^2}{s_{\mathrm{min}}^2 - \kappa^2}} 
	= \frac{1 + \ln(s_0) - \ln(s_{\mathrm{min}})}{2\delta + \frac{\kappa^2}{s_{\mathrm{min}}^2 - \kappa^2}} . 
	\end{align*}\qedhere
\end{proof}

We conclude this section with the following lemma on the occupation probability of a random walk between two states.

\begin{lemma}\label{lem:twostates}
Let $\delta \geq 0$, and let $r \geq b \geq 0$. Consider a time-discrete random walk $(X_t)_{t\in\N}$ with two states $A$ and $B$, adapted to some filtration $\filtration{F}_t$. For any $t\geq 0$, let $S_t \coloneqq \min\{t' \geq 0  \mid  X_{t+t'} = A\}$ be the number of rounds to reach $A$ for the next time after $t$. Suppose that
\begin{enumerate}
\item $\Pr[X_{t+1} = B \mid \filtration{F}_t, X_{t}=A] \geq \delta$ for all $t\geq 0$.
\item There exists $s \geq 0$ such that for all $t\geq 0$,
\[
\Pr[S_t \geq s \mid \filtration{F}_t, X_t = B, X_{t-1}=A] \geq \frac{b}{s}.
\]
\end{enumerate}
Then, if $N_A(r) \coloneqq |\{1 \leq t \leq r \mid X_t=A\}|$ denotes how many of the first $r$ round we spend in $A$, we have
\[
\EE[N_A(r)] \leq \frac{2r}{b\delta},
\]
and
\[
\Pr\left[N_A(r) > \frac{4r}{b\delta}\right] \leq e^{-r/(2s)}.
\]

\end{lemma}
We remark that Condition (2) cannot be replaced by the weaker condition $\EE[S_t \mid \filtration{F}_t, X_t = B, X_{t-1}=A] \geq b$, not even for the statement on the expectation. For example, for $r \gg b \gg 1$ set $S_t \coloneqq r^2$ with probability $b/r^2$, and $S_t \coloneqq 1$ otherwise. Then by a union bound, with probability $\Omega(1)$ we never observe $S_t = r^2$ in the first $r$ rounds, so $\EE[N_A(r)] \in \Omega(r)$.
\begin{proof}[Proof of Lemma~\ref{lem:twostates}]
We first consider the case that $r=s$. We claim that $N_A(s)$ is stochastically dominated by a geometric random variable~$\text{Geo}(p)$, where $p \coloneqq \delta b/s$. Consider the first $r=s$ rounds. By condition (2), whenever we enter $B$, we spend all the remaining rounds in $B$ with probability at least $b/s$. We pessimistically assume that we immediately return to $A$ otherwise. Then for $X_t =A$, one of the following three cases will happen.
\begin{enumerate}
\item $X_{t+1} = A$, with probability at most $1-\delta$.
\item $X_{t+1}=B$ and $X_{t+2} = A$, with probability at most $\delta (1-b/s)$.
\item $X_{t+1}=X_{t+2} = \ldots, X_r =b$, with probability at least $\delta b/s$.
\end{enumerate}
Hence, $N_A(s)$ is stochastically dominated by~$\text{Geo}(p)$ as claimed. In particular, $\EE[N_A(s)] \leq 1/p = s/(b\delta)$.

For the other case $r>s$, we split up the random walk into $k \coloneqq \lceil r/s\rceil$ phases of length $s$ each, which covers slightly more than $r$ rounds. Then in each phase we know that the expected number of rounds in $A$ is dominated by~$\text{Geo}(\delta b/s)$. Regarding the expectation, the total number of rounds in $A$ is at most $\EE[N_A(r)] \leq k\cdot s/(b\delta) \leq 2r/(b\delta)$. For the tail bound, we need to bound the probability $q \coloneqq \Pr[Y_1 + \ldots + Y_k > {4r}/(b\delta)]$, where the $Y_i$ are independent random variables with distribution~$\text{Geo}(p)$. We equivalently characterize $q$ by $q= \Pr[\text{Bin}({4r}/(b\delta), p) < k]$. Since $k < 2r/s = \tfrac12 {4rp}/(b\delta)$, from the Chernoff bound, Theorem~\ref{thm:Chernoff}, we deduce $q \leq e^{-(1/2)^2(4r/s)/2} = e^{-r/(2s)}$.
\end{proof}

\section{Results with Bloat Control}
\label{sec:bloatControl}
In this section we show the following theorem.
\begin{theorem}\label{thm:bloatControl}
The \oneonegp \emph{with bloat control} choosing $k = 1 + \Pois(1)$ on \order and  \majority takes $\Theta(\tinit + n \log n)$ iterations in expectation.
\end{theorem}

\subsection{Lower Bound}

Regarding the proof of the lower bound, let $\tinit$ and $n$ be given. Let $t$ be a \gptree which contains $\tinit$ leaves labeled $\nonvar_1$. From a simple coupon collector's argument we get a lower bound of $\Omega(n \log n)$ for the run time to insert each $\var_i$. As an optimal tree cannot list any of the leaves in $t$ in addition to the expected number of deletions performed by \oneonegp being in $\BigO(1)$, we obtain a lower bound of $\tinit$ from the additive drift theorem (Theorem~\ref{thm:additiveDrift}).

\subsection{Upper Bound}
This section is dedicated to the proof of the upper bound. Let $t$ be a \gptree over $n$ variables and denote the number of expressed variables of $t$ by $v(t)$. We call the number of leaves of $t$ the \emph{size} of $t$ and denote it by $s(t)$. For a best-so-far \gptree of the \oneonegp we denote the size of the initial \gptree by $\tinit$. Both parameters $n$ and $\tinit$ are considered to be given.
The main difference to the case of only one mutation per iteration of the \oneonegp is that with more mutations in a single iteration the number of expressed variables can increase together with the introduction of a number of redundant leaves. The increased fitness will hinder the bloat control from rejecting the offspring even though the size could have increased by a large amount. 

In order to deal with this behavior we are going to partition the set of leaves by observing the change of fitness when deleting one leaf. For a redundant leaf, the fitness is not affected by deleting it. However, not every non-redundant leaf contributes an expressed variable, since the deletion of a leaf can also increase the fitness if it is a negative literal. Thus, we consider the following sets of leaves.
\begin{center}
\begin{tabularx}{.95\textwidth}{@{}l X}
	$R(t)$: & \emph{Redundant leaves} $v$, where the fitness of $t$ is not affected by deleting $v$. \\
	$C^+(t)$: & \emph{Critical positive leaves} $v$, where the fitness of $t$ decreases by deleting $v$.\\
	$C^-(t)$: & \emph{Critical negative leaves} $v$, where the fitness of $t$ increases by deleting $v$.
\end{tabularx}
\end{center}
We denote by $r(t)$, $c^+(t)$ and $c^-(t)$ the cardinality of $R(t)$, $C^+(t)$ and $C^-(t)$, respectively. Thus we obtain
\begin{equation}\label{eq:proof_thm7_st}
s(t)=r(t) + c^+(t) + c^-(t).
\end{equation}

The general idea of the proof is the following: We are going to construct a suitable potential function $g$ mapping a \gptree $t$ to a natural number in such a way that the optimum receives a value of $0$ and the function displays the fitness with respect to the number of expressed variables and the size in a proper way. For a best-so-far \gptree $t$ let $t'$ be the offspring of $t$ under the \oneonegp. By bounding the drift, i.e. the expected change $g(t) - g(t')$ denoted by $\Delta (t)$, we are going to obtain the bound for the optimization time due to Theorem~\ref{thm:variableDrift}.

Regarding the bound on the drift we already argued that the case of only one mutation in an iteration is beneficial, since either the amount of expressed variables of parent and offspring are the same or the offspring has exactly one more variable expressed. However, the case of at least two mutations in an iteration is problematic in the above mentioned sense. In order to deal with the negative drift (leading away from the optimum) introduced by the latter case, the positive drift due to the other case has to outweigh the negative drift. Therefore, we need to bound the drift in both cases carefully.

We observe that starting with a very big initial tree the algorithm will delete redundant leaves with a constant probability until most of the occurring variables are expressed. In this second stage the size of the tree is at most linear in $n$ and the algorithm will insert literals, which do not occur in the tree at all, with a probability of at least linear in $1/n$ until all variables are expressed. In order to obtain a better bound on the drift, we will split the second stage in two cases. Finally, by the law of total expectation we will obtain a bound on the drift due to the bounds under the mentioned cases.

%
In order to deal with critical leaves, we are going to prove upper bounds on the number of these. In fact, there exists a strong correlation between critical and redundant leaves we are going to exploit frequently.

\begin{lemma}\label{lem:bound_crit_leav}
	Let $t$ be a \gptree, then for \order and \majority we have
	\begin{enumerate}[label=(\roman*)]
		\item $c^+(t) \leq r(t) + v(t)$,
		\item $c^-(t) \leq 2 r(t)$.
	\end{enumerate}
\end{lemma}
\begin{proof}
	We proof both statements by observing the behavior of \order and \majority individually.
	\begin{enumerate}[wide, labelwidth=!, labelindent=0pt, parsep=0pt]
	\item[(i):] $\mbox{}$\\
		Let $opt(t)$ be the number of optimal leaves, i.e. positive leaves $\var_i$, where no additional instances of the variable $i$ are present in $t$. Obviously $opt(t) \leq v(t) \leq n$ holds. We observe		
		\[
		c^+(t) - v(t) \leq c^+(t) - opt(t),
		\] 
		thus it suffices to bound the number of non-optimal critical positive leaves.
		
		For \majority a variable $i$ can only contribute such a leaf, if the number of positive literals $\var_i$ equals the number of negative literals $\nonvar_i$. Since every such negative literal is a redundant leaf, we obtain $c^+(t) - opt(t) \leq r(t)$.
		
		For \order a variable $i$ can only contribute such a leaf, if the first occurrence of $i$ is a positive literal $\var_i$ and the second occurrence is a negative literal $\nonvar_i$. In this case the negative literal as well as every additional occurrence of a literal $\var_i$ is a redundant leaf. Therefore, we deduce $c^+(t) - opt(t) \leq r(t)$.
	\item[(ii):] $\mbox{}$\\
		For \majority a variable $i$ can only contribute a critical negative leaf if the number of positive literals $\var_i$ is $m$ and the number of negative literals $\nonvar_i$ is $m+1$ for some $m \geq 1$. In this case each negative literal is a critical negative leaf and each positive literal is a redundant leaf. We obtain $c^-(t) \leq 2 r(t)$.
				
		For \order a variable $i$ can only contribute a critical negative leaf if the first occurrence of $i$ is a negative literal and the second occurrence is a positive literal. In this case the first occurrence is a critical negative leaf and every additional occurrence afterwards is a redundant leaf. We obtain $c^-(t) \leq r(t)$. \qedhere
	\end{enumerate}
\end{proof}
In order to construct the mentioned potential function, we want to reward strongly an increase of fitness given by a decrease of the unexpressed variables. Furthermore, we want to reward a decrease of size but without punishing an increase of fitness. Here, we need to be careful with the weights for both changes since a strong reward for a decrease of size might result in a very big negative drift in case of at least two operations. In order to illustrate the choice for the weights, we will fix the weight $m \in \mathbb{R}_{>0}$ for a decrease of unexpressed variables only later on. Thus, we associate with $t$ the potential function 
\[
g(t) = m(n- v(t)) + s(t) - v(t).
\]
This potential is $0$ if and only if $t$ contains no redundant leaves and for each $i \leq n$ there is an expressed $\var_i$. Furthermore, by Lemma~\ref{lem:bound_crit_leav} $s(t) - v(t)$ is also $0$ since $r(t)$ is $0$.

Let $\event{D}_1$ be the event where the algorithm chooses to do exactly one operation in the observed mutation step, and $\event{D}_2$ where the algorithm chooses to do at least two operations in the observed mutation step.
Since the algorithm chooses in each step at least one operation, we observe
\begin{align*}
\Prob{\event{D}_1} &= \Prob{\Pois(1)=0}=\frac{1}{e} , \\
\Prob{\event{D}_2} &= 1- \frac{1}{e}. 
\end{align*}

Now we are going to derive bounds on the negative drift in the case $\event{D}_2$. These are going to be connected with bounds on the positive drift for $\event{D}_1$ by the law of total expectation.
Let $\event{E}$ be the event that $v(t') = v(t)$. As argued above, in the case $\event{E}$ the potential cannot increase even if $\event{D}_2$ holds. However, conditional on $\overline{\event{E}}$ the potential can increase yielding a negative drift.

\begin{lemma}\label{lem:bounds_neg_drift_thm7}
	For the expected negative drift measured by $g(t)$ conditional on $\event{D}_2$ holds
	\begin{equation*}
		\Ex{\Delta (t) \mid \event{D}_2} \geq - \frac{1}{e} \left( 2e - me  + \sum_{i=1}^{m} \frac{m-i}{(i-1)!} \right). 
	\end{equation*}
	In addition, if $s(t) > n/2$ holds, this bound is enhanced to
	\begin{equation*}
		\Ex{\Delta (t) \mid \event{D}_2} > - \frac{g(t)}{en} \left( \frac{1}{6m} + \frac{2}{3}\right) \left(2e - 5me + \sum_{i=1}^{m} \frac{i(m-i)}{(i-1)!}\right). 
	\end{equation*}
\end{lemma}
\begin{proof}
	Concerning the drift conditional on $\event{D}_2$ we observe 
	\begin{equation}\label{eq:proof_thm7_negative_drift_total_exp}
		\Ex{\Delta (t) \mid \event{D}_2} \geq  - \Ex{- \Delta (t) \mid \overline{\event{E}}}~\Prob{\overline{\event{E}}},
	\end{equation}
	since the drift can be negative only in this case. In particular, we observe a drift of at least $m$ for the increase of fitness counteracted by the possible increase of the size. The latter is at most the number of operations the algorithm does in the observed step, because every operation can increase the size by at most $1$.
	
	Let $Y \sim \Pois(1)+1$ be the random variable describing the number of operations in a round. Note that, for all $i \geq 1$,
	\[
		\Prob{Y=i} = \frac{1}{e(i-1)!}.
	\]
	By this probability we obtain for the expected negative drift conditional on $\overline{\event{E}}$
	\begin{align*}
		\Ex{-\Delta(t) \mid \overline{\event{E}}} &= \sum_{i=0}^\infty \Ex{-\Delta(t) \mid Y = i, \overline{\event{E}}}~\Prob{Y=i \mid \overline{\event{E}}} \leq \sum_{i=0}^\infty (i-m) ~ \Prob{Y=i \mid \overline{\event{E}}}\\
		&\leq \sum_{i=m+1}^\infty (i-m) ~ \Prob{Y=i \mid \overline{\event{E}}}.
	\end{align*}
	Due to Bayes' theorem we derive
	\begin{equation}\label{eq:proof_thm7_expected_negative_drift}
		\nonumber
		\Ex{-\Delta(t) \mid \overline{\event{E}}} \leq \sum_{i=m+1}^\infty (i-m) ~ \Prob{\overline{\event{E}} \mid Y=i} ~\frac{\Prob{Y=i}}{\Prob{\overline{\event{E}}}},
	\end{equation}
	which yields the first bound due to inequality~\eqref{eq:proof_thm7_negative_drift_total_exp} by pessimistically assuming $\Prob{\overline{\event{E}} \mid Y=i} = 1$
	\begin{align}\label{eq:proof_thm7_bound_neg_drift_const}
		\nonumber
		\Ex{\Delta (t) \mid \event{D}_2} &\geq - \sum_{i=m+1}^\infty (i-m) ~ \Prob{Y=i} = - \frac{1}{e} \left( 2e - me  + \sum_{i=1}^{m} \frac{m-i}{(i-1)!} \right). 
	\end{align}
	
	In order to obtain a better bound on the negative drift, we are going to bound the probability $\Prob{\overline{\event{E}} \mid Y=i}$ by a better bound than the previously applied bound of $1$.
	
	The event $\overline{\event{E}}$ requires a non-expressed variable in $t$ to become expressed in $t'$. There are $n-v(t)$ non-expressed variables in $t$. These can become expressed by either adding a corresponding positive literal or deleting a corresponding negative literal. There are $2n$ literals in total and due to $n-v(t) \leq g(t)/m$ adding such a positive literal has a probability of at most
	\[
		\frac{n-v(t)}{6n} \leq \frac{g(t)}{6mn}
	\]	
	per operation. Regarding the deletion of negative literals, there are at most $s(t) - v(t)$ negative literals. Hence, due to $s(t) - v(t) \leq g(t)$ and $s(t) > n/2$ the probability of deleting a negative literal is at most
	\[
		\frac{s(t) - v(t) }{3 s(t)} \leq \frac{2 g(t)}{3 n}
	\]
	per operation. Let $q_l$ be the probability that the $l$-th mutation leads an unexpressed variable to become expressed. We can bound the probability that $i$ operations lead to the expression of a previously unexpressed bound by pessimistically assuming that the mutation is going to be accepted. This yields by the union bound
	\begin{align*}
		\Prob{\overline{\event{E}} \mid Y=i} &\leq \bigcup_{l=1}^{i} q_l \leq \sum_{l=1}^{i} q_i = \frac{i g(t)}{n} \left( \frac{1}{6m} + \frac{2}{3}\right) .
	\end{align*}
	Therefore, we obtain due to inequality~\eqref{eq:proof_thm7_negative_drift_total_exp} an expected drift conditional on $\event{D}_2$ of
	\begin{align*}
		\Ex{\Delta(t) \mid \event{D}_2} &  
		> - \frac{g(t)}{en} \left( \frac{1}{6m} + \frac{2}{3}\right) \sum_{i=m+1}^{\infty} \frac{i(i-m)}{(i-1)!} \\
		&= - \frac{g(t)}{en} \left( \frac{1}{6m} + \frac{2}{3}\right) \left(2e - 5me + \sum_{i=1}^{m} \frac{i(m-i)}{(i-1)!}\right). 
	\end{align*}
\end{proof}
As a small spoiler for the choice of $m$, we will give the following Corollary on Lemma~\ref{lem:bound_crit_leav}.
\begin{corollary}\label{cor:bc_neg_drift}
	For $m=10$ we obtain the following bounds 
	\begin{equation*}
		\Ex{\Delta (t) \mid \event{D}_2} \geq - \frac{1}{e} \left( 4 \cdot 10^{-7} \right).
	\end{equation*}
		In addition, if $s(t) > n/2$ holds, this bound is enhanced to
	\begin{equation*}
		\Ex{\Delta (t) \mid \event{D}_2} > - \frac{7 g(t)}{10 en} \left( 4 \cdot 10^{-6} \right).
	\end{equation*}
\end{corollary}

We are now going to prove the upper bound by deriving the expected positive drift outweighing the negative drift given by Lemma~\ref{lem:bounds_neg_drift_thm7}.
	
\textbf{Case 1:}
	We first consider the case $r(t) \geq v(t)$. Due to Lemma~\ref{lem:bound_crit_leav} and Equation~\eqref{eq:proof_thm7_st} we obtain
	\[
	s(t) = r(t) + c^+(t) + c^-(t) \leq 4r(t) + v(t) \leq 5 r(t),
	\]
	thus the algorithm has a probability of at least $1/5$ for choosing a redundant leaf followed by choosing a deletion with probability $1/3$. Since the deletion of a redundant leaf without any additional operations does not change the fitness this contributes to the event $\event{E}$. Hence, we obtain for the event $\event{D}_1$
	\[
		\Ex{\Delta (t) \mid \event{D}_1 ,\, \event{E}} ~ \Prob{\event{E}} \geq \frac{1}{15}.
	\]
	Additionally, the drift conditional on $\event{D}_1$ is always positive, which yields
	\begin{equation*}
		\Ex{\Delta (t) \mid \event{D}_1} \geq \Ex{\Delta (t) \mid \event{D}_1 ,\, \event{E}} ~ \Prob{\event{E}} \geq \frac{1}{15}.
	\end{equation*}
	
	The drift conditional on $\event{D}_2$ is given by Lemma~\ref{lem:bounds_neg_drift_thm7}. We observe, that the positive drift of $1/15$ outweighs the negative drift for the choice of $m=10$ given by Corollary~\ref{cor:bc_neg_drift}.
	Overall, we obtain a constant drift in the case of $r(t) \geq v(t)$ due to the law of total expectation
	\begin{align}
		\label{eq:bc_drift_case1}
		\nonumber
		\Ex{\Delta (t) } &\geq \Ex{\Delta (t) \mid \event{D}_1} ~ \Prob{\event{D}_1} + \Ex{\Delta (t) \mid \event{D}_2} ~ \Prob{\event{D}_2} \geq \frac{1}{15e} - \frac{1}{e}\left(1 - \frac{1}{e}\right) \left(4 \cdot 10^{-7} \right) \\
		&\geq \frac{1}{e} \left( \frac{1}{15} - 4 \cdot 10^{-7} \right) 
		\geq \frac{3}{50e}.
	\end{align}

	\textbf{Case 2:}
	Suppose $r(t) < v(t)$ and $s(t) \leq n/2$. In particular, we have for at least $n/2$ many $i \leq n$ that there is neither $x_i$ nor $\nonvar_i$ present in $t$. The probability to choose $\var_i$ is at least $n/4$ and the probability that the algorithm chooses an insertion is $1/3$. This insertion will yield a fitness increase of $m$ and since the location of the newly inserted literal is unimportant we obtain
	\[
		\Ex{\Delta(t) \mid \event{D}_1} ~ \Prob{\event{D}_1} \geq \frac{m}{12e}.
	\]
	For the expected drift in the case $\event{D}_2$ holds we apply again the bound given by Lemma~\ref{lem:bounds_neg_drift_thm7}. Analogue to Case 1 we observe, that the positive drift outweighs the negative drift for the choice of $m=10$, which yields the following constant drift
	\begin{align*}
		\Ex{\Delta(t)} &\geq \frac{1}{e} \left( \frac{10}{12} - 4 \cdot 10^{-7} \right) > \frac{8}{10e}.
	\end{align*}
	
	\textbf{Case 3:}
	Consider now the case that $r(t) < v(t)$ and $s(t) > n/2$.
	In particular, the tree can contain at most $5n$ leaves due to
	\[
	s(t)\leq 4r(t) + v(t) < 5v(t) \leq 5n,
	\]
	which enables us to bound the probability that an operation chooses a specific leaf $v$ as
	\[
	\frac{1}{5n} \leq \Prob{\mbox{choose leaf }v} \leq \frac{2}{n}.
	\]
	
	Let $A$ be the set of $i$, such that there is neither $\var_i$ nor $\nonvar_i$ in $t$, and let $B$ be the set of $i$, such that there is exactly one $\var_i$ and no $\nonvar_i$ in $t$. Recall that $R(t)$ is the set of redundant leaves in $t$. For every $i$ in $A$ let $\event{A}_i$ be the event that the algorithm adds $\var_i$ somewhere in $t$. For every $j$ in $R(t)$ let $\event{R}_j(t)$ be the event, that the algorithm deletes $j$. Finally, let $\event{A}'$ be the event that one of the $A_i$ holds, and $\event{R}'$ the event that one of the $\event{R}_j(t)$ holds.
	
	Conditional on $\event{D}_1$ we observe for every event $\event{A}_i$ a drift of $m$. For each event $\event{R}_j(t)$ conditional on $\event{D}_1$ we observe a drift of $1$ since the amount of redundant leaves decreases by exactly $1$. Hence,
	\begin{align*}
		\Ex{\Delta (t) \mid \event{A}_i,\, \event{D}_1} &= m , \\
		\Ex{\Delta (t) \mid \event{R}_j(t),\, \event{D}_1} &= 1 .
	\end{align*}
	
	Regarding the probability for these events we observe that for $\event{A}_i$ the algorithm chooses with probability $1/3$ to add a leaf and with probability $1/(2n)$ it chooses $\var_i$ for this. Furthermore, the position of the new leaf $\var_i$ is unimportant, hence
	\begin{equation*}
		\Prob{\event{A}_i \mid \event{D}_1} \geq \frac{1}{6n} .
	\end{equation*}
	Regarding the probability of $\event{R}_j(t)$, with probability at least $1/(5n)$ the algorithm chooses the leaf $j$ and with probability $1/3$ the algorithm deletes $j$. This yields
	\begin{equation*}
		\Prob{\event{R}_j(t) \mid \event{D}_1} \geq \frac{1}{15 n} .
	\end{equation*}
	
	In order to sum the events in $\event{A}'$ and $\event{R}'$, we need to bound the cardinality of the two sets $A$ and $R(t)$. For this purpose we will need the above defined set $B$. First we note that the cardinality of $B$ is at most $v(t)$. In addition
	\begin{equation}\label{eq:proof_thm7_bound_ARr}
	|A|+|R(t)| \geq r(t)
	\end{equation}
	holds since $R(t)$ is the set of all redundant leaves. Furthermore, we observe that for any variable $j$, which is not in $B$ or $A$, there has to exist at least one redundant leaf $\var _j$ or $\nonvar_j$. Since every redundant leaf is included in $R(t)$ we obtain $|A| + |R(t)| + |B| \geq n$
	and subsequently 
	\begin{equation}\label{eq:proof_thm7_bound_ARnv}
	|A| + |R(t)| \geq n - v(t) .
	\end{equation}
	Furthermore, due to Lemma~\ref{lem:bound_crit_leav} we deduce
	\begin{align}\label{eq:proof_thm7_bound_ABst}
		s(t) - v(t) &\leq r(t) + c^+(t) + c^-(t) - v(t) 
		\leq 4r(t) \leq 4( |A| + |R(t)|),
	\end{align}
	where the last inequality is due to \eqref{eq:proof_thm7_bound_ARr}. This inequality \eqref{eq:proof_thm7_bound_ABst} in conjunction with \eqref{eq:proof_thm7_bound_ARnv} yields
	\begin{equation}\label{eq:proof_thm7_bound_ARg}
	(m+4)(|A| + |R(t)|) \geq  m (n - v(t)) + s(t) - v(t)  = g(t).
	\end{equation}
	
	We obtain the expected drift conditional on the event $\event{D}_1$ as for $m\geq1$
	\begin{align*}
		\Ex{\Delta (t) \mid \event{D}_1} 
		&\geq \Ex{\Delta (t) \mid (\event{A}' \vee \event{R}'),\, \event{D}_1 } ~ \Prob{\event{A}' \vee \event{R}' \mid \event{D}_1} \\
		&= \sum_{i \in A} \Ex{\Delta (t) \mid \event{A}_i,\, \event{D}_1} ~ \Prob{\event{A}_i ,\, \event{D}_1} + \!\! \sum_{j \in R(t)} \!\!\! \Ex{\Delta (t) \mid \event{R}_j(t),\, \event{D}_1} ~ \Prob{\event{R}_j(t) \mid \event{D}_1} \\
		&\geq |A| \frac{m}{6n} + |R(t)| \frac{1}{15 n} \geq \left( |A| + |R(t)| \right) \frac{1}{15 n}  \geq \frac{ g(t)}{15(m+4) n},
	\end{align*}
	where the last inequality is due to \eqref{eq:proof_thm7_bound_ARg}.
	Concerning the expected drift conditional on $\event{D}_2$, the condition for the second bound given by Lemma~\ref{lem:bounds_neg_drift_thm7} is satisfied in this case. Again, we observe that the positive drift outweighs the negative drift for $m=10$ given by Corollary~\ref{cor:bc_neg_drift}, which justifies the choice of $m=10$ we are setting from here on. In fact, we could choose any integer $m \geq 5$ in order for the positive drift to outweigh the negative.
	Summarizing the events $\event{D}_1$ and $\event{D}_2$ we obtain the expected drift
	\begin{align}
		\label{eq:bc_drift_case3}
		\nonumber
		\Ex{\Delta(t)} &\geq \Ex{\Delta(t) \mid \event{D}_1} ~ \Prob{\event{D}_1} + \Ex{\Delta(t) \mid \event{D}_2} ~ \Prob{\event{D}_2} \\
		&\geq \frac{g(t)}{e n } \left( \frac{ 1}{210 } - \left(1 - \frac{1}{e }\right) \frac{7 }{10 } \cdot 4 \cdot 10^{-6} \right) 
		> \frac{ g(t)}{250 e n } .
	\end{align}
	
	Summarizing the derived expected drifts \eqref{eq:bc_drift_case1} and \eqref{eq:bc_drift_case3}, we observe a multiplicative drift in the case of
	\[
	\frac{ g(t)}{250e n} \leq \frac{3}{50e} ,
	\]
	which simplifies to $g(t) \leq 15n$. If $g(t) > 15n$, we observe a constant drift. This constant drift is at least $3/50e$ since the expected drift for Case 2 is always bigger than the one for Case 1.
	
	We now apply the variable drift theorem (Theorem~\ref{thm:variableDrift}) with $h(x)= \min \{3/(50e) , \\ \,  1x / (250e n)\}$, $X_0 = \tinit + 10n$ and $X_{\min} =1$, which yields
	\begin{align*}
		\Ex{T &\mid g(t)=0}  \leq \frac{1}{h(1)} + \int_{1}^{\tinit + 10n} \frac{1}{h(x)} ~dx\\
		& = 250en + 250en \int_{1}^{15n} \frac{1}{x} ~dx + \frac{50e}{3} \int_{15n+1}^{\tinit + 10n} 1 ~dx\\		
		&= 250e n \left(1 + \log (15n) \right) + \frac{50e}{3} \left(\tinit -5n  -1 \right) < 250en\log (15e n ) + \frac{50e}{3} \tinit.
	\end{align*}
	
	This establishes the theorem.

\section{Results Without Bloat Control}
\label{sec:majority}
In this section we show the following theorems.

\begin{theorem}\label{thm:majority_lbound}
	The \oneonegp \emph{without bloat control} (choosing $k=1$ or $k = 1 + \Pois(1)$) on MAJORITY takes 
	$\Omega(\tinit\log \tinit)$ iterations in expectation for $n=1$. For general $n\geq 1$ it takes 
	$\Omega(\tinit + n\log n)$ iterations in expectation.	
\end{theorem}

\begin{theorem}\label{thm:majority}
The \oneonegp \emph{without bloat control} (choosing $k=1$ or $k = 1 + \Pois(1)$) on MAJORITY takes $\BigO(\tinit\log \tinit  + n \log^3 n)$ iterations in expectation.
\end{theorem}

\subsection{Proof of the Lower Bound}

Regarding the proof of Theorem~\ref{thm:majority_lbound}, let $\tinit$ be large. Let $t_0$ be a \gptree which contains $\tinit$ leaves labeled $\nonvar_1$ and no other leaves. From a simple coupon collector's argument we get a lower bound of $\Omega(n \log n)$ for the run time to insert each $\var_i$. It remains to bound the time the algorithm needs to express the variable $1$. 

In order to derive the bound for general $n\geq 1$ we observe, that the algorithm does in expectation $2$ operations in each iteration since $\Ex{1+\Pois(1)}=2$. Hence, the algorithm needs in expectation at least $\tinit /2$ iterations to express the first variable yielding the desired result.

Regarding the bound for the case $n=1$ let $t$ be a \gptree, let $I_1(t)$ be the number of literals $\var_1$ in $t$ and $I'_1(t)$ be the number of literals $\nonvar_1$ in $t$. We associate with $t$ the potential function $g(t)$ by
\[
g(t)=I'_1(t) - I_1(t).
\]
In order to express the variable $1$, the potential $g(t)$ has to become non-negative at one point. In particular, starting with $g(t_0) = \tinit$, the potential has to reach a value of at most $\tinit^{2/3}$. Let $\tau$ denote the number of iterations until the algorithm encounters for the first time a \gptree $t$ with $g(t) \leq \tinit^{2/3}$. We are going to bound the expected value of $\tau$ starting with $t_0$, since this will yield a lower bound for the expected number of iterations until the variable $1$ is expressed.

Let $\event{A}_i$ be the event, that the algorithm performs more than $15\ln (\tinit)$ operations in the $i$-th iteration. For a better readability we define $z$ to be $15\ln (\tinit)$. Regarding the probability of $\event{A}_i$ we obtain due to the Poisson-distributed number of operations
\begin{align*}
\Prob{\event{A}_i} = \sum_{i=z}^{\infty} \frac{1}{e (i-1)!} .
\end{align*}
Let $p_i$ be the probability, that a $\Pois(1)$ distributed random variable is equal to $i$. We derive
\[
p_{i+1} = p_i \frac{1}{i+1} \leq p_i \frac{1}{2}.
\]
Since $\event{A}_i$ is $\Pois(1)$-distributed, this yields
\[
\Prob{\event{A}_i} \leq p_z \sum_{i=0}^{\infty} \frac{1}{2^i} = \frac{2}{ez!}.
\]
By the Stirling bound $n! \geq e (n/e)^n$ we obtain
\[
\Prob{\event{A}_i} \leq \frac{e^z}{ez^z} \leq \frac{\tinit ^{15}}{z^z} \leq \tinit ^{-15},
\]
where the last inequality comes from $z^z \geq e^{2z}$, which holds for $\tinit \geq 2$.

Let $\event{A}$ be the event that in $\tinit ^2$ iterations the algorithm performs at least once more than $z$ operations in a single iterations. By the union bound we obtain for the probability of $\event{A}$
\[
\Prob{\event{A}} = \Pr\left[\bigcup_{i=1}^{\tinit^2} \event{A}_i \right] \leq \sum_{i=1}^{\tinit^2} \Prob{\event{A}_i} \leq \tinit^{-13}.
\]
Hence, w.h.p.\ the algorithm will not encounter the event $\event{A}$. By the law of total expectation we deduce
\[
\Ex{\tau} = \Ex{\tau \mid \event{A}} ~ \Prob{\event{A}} + \Ex{\tau \mid \overline{\event{A}}} ~ \Prob{\overline{\event{A}}} \geq \Ex{\tau \mid \overline{\event{A}}} \frac{1}{2}.
\]
It remains to bound the expected value of $\tau$ under the constraint of $\overline{\event{A}}$.

Let $t'$ be the random variable describing the best-so-far solution in the iteration after $t$. We are going to bound the drift, i.e. the expected change $g(t)-g(t')$, which we denote by $\Delta(t)$. We recall that $g(t)=I'_1(t) - I_1(t)$, where $I'_1(t)$ is the number of literals $\nonvar_1$ and $I_1(t)$ is the number of literals $\var_1$. If the algorithm chooses an insertion, the probability to insert $\var_1$ is the same as the probability to insert $\nonvar_1$. Therefore, an insertion will only contribute $0$ to the expected drift. The same holds for the literals \emph{introduced} by a substitution. However, for literals \emph{deleted} by a deletion or substitution the probability to choose a literal $\var_1$ or $\nonvar_1$ is of importance contrary to an insertion.

Let $\event{B}$ be the event, that the algorithm chooses at least once a literal $\var_1$ for a substitution or deletion in this iteration. The probability of $\event{B}$ is at least the probability for the algorithm to do exactly one operation: a deletion or substitution of a literal $\var_1$. Let $s(t)$ be the amount of leaves of $t$ (the \emph{size}). We deduce
\[
\Prob{\event{B}} \geq \frac{2}{3e} \, \frac{I_1(t)}{s(t)}.
\] 
Furthermore, the expected negative drift of $g(t)$ can be bounded by this event $\event{B}$, which yields
\[
\Ex{\Delta \mid \event{B}} = -1.
\]

Regarding the positive drift, let $\event{C}_i$ be the event, that in this iteration the algorithm chooses to do $i$ operations, which are either substitutions or deletions of literals $\nonvar_1$. Again, the algorithm chooses with probability $1/3$ to do a substitution. Additionally, the algorithm chooses to do $i$ operations with probability $p_{i-1}$ with $p_i$ as defined above. However, the probability to choose a literal $\nonvar_1$ changes with each operation. Each deletion of a literal $\nonvar_1$ reduces $s(t)$ and $I'_1$ by $1$. Each substitution of a literal $\nonvar_1$ reduces $s(t)$ by $1$ and $I'_1$ stays the same. Therefore, we can bound the probability for a substitution by at most the probability of a deletion. This yields for $I'_1(t) < s(t)$ 
\[
\Prob{\event{C}_i} \leq \frac{2}{3^i}  \, p_{i-1}  \, \frac{I'_1(t)! (s(t)-i)!}{s(t)!(I'_1(t)-i)!} \leq \frac{2}{3^i} \, p_{i-1} \, \frac{I'_1(t)}{2 s(t)}.
\]
Hence, we obtain the expected drift for $\overline{\event{B}}$
\begin{align*}
\Ex{\Delta (t) \mid \overline{\event{B}}} ~ \Prob{\overline{\event{B}}}& \leq \frac{I'_1(t)}{e s(t)} \sum_{i=1}^{\infty} \frac{i}{3^i (i-1)!} = \frac{ 4I'_1(t)}{9 e^{2/3} s(t)}.
\end{align*}
Summarizing, we obtain by the law of total expectation
\[
E(\Delta (t)) \leq \frac{ 4 I'_1(t)}{9 e^{2/3} s(t)} - \frac{2I_1(t)}{3es(t)} \leq  \frac{2g(t)}{3e s(t)} .
\]

In order to bound the size $s(t)$ we observe that following a standard gambler's ruin argument within $\LittleO(\tinit^{1.5})$ iterations the size will not shrink by a factor bigger than $1/2$.
Therefore, we obtain $s(t) \geq 1/2 ~\tinit$.
Due to the step size bound of $15\ln (\tinit) < \tinit^{2/3}$ we can apply Theorem~\ref{thm:Multi_Drift_Bounded_Step_Size} and derive
\begin{align*}
\Ex{\tau \mid \overline{\event{A}}, \, X_0=\tinit} &\geq \frac{1 + \ln (\tinit) - \ln (\tinit ^{1/2})}{\frac{2}{3e\tinit} + \frac{(15\ln(\tinit))^2}{\tinit^{4/3} - (15\ln(\tinit))^2}}.
\end{align*}
In order to simplify this bound we observe $\ln(\tinit) \leq 3 \tinit ^{1/3}$, which yields
\[
\frac{(15\ln(\tinit))^2}{\tinit^{4/3} - (15\ln(\tinit))^2} 
\leq \frac{(15\ln(\tinit))^2}{\tinit^{4/3} - (45 \tinit^{1/3})^2}
\leq \frac{1}{2 \tinit} .
\]
Therefore, we obtain
\[
\Ex{\tau} \geq \frac{3e \, \tinit \ln (\tinit)}{8+12e}
\]
establishing the theorem.
\subsection{Proof of the Upper Bound}

\subsubsection{Outline}\label{sec:upperboundoutline}

Since the proof of Theorem~\ref{thm:majority} is long and involved, we first give an outline of the proof. The key ingredient is a bound on the bloat, i.e., on the speed with which the tree grows. Roughly speaking, we will show in Theorem~\ref{thm:lengthMA} that if $\tinit \geq n\log^2 n$, then the size of the tree grows at most by a constant factor in $\BigO(\tinit \log \tinit)$ rounds. 

Before we elaborate on the bloat, let us first sketch how this implies the upper bound. Consider any $\var_i$ that is not expressed and let $V'(t_r,i) \coloneq \#\{\nonvar_i\text{-literals}\}-\#\{\var_i\text{-literals}\}\geq 1$. (For this outline we neglect the case that there are neither $\nonvar_i$ nor $\var_i$ in the string.) Then the probability of deleting or relabeling a $\nonvar_i$ is larger than deleting or relabeling a $\var_i$, while they have the same probability to be inserted. Computing precisely, denoting $t_r$ the \gptree in round $r$, we get a drift
\begin{equation}\label{outlineeq:MA}
\EE[V'(t_{r+1},i)-V'(t_{r},i) \mid V(t_{r},i) = v] \leq -\frac{v}{3e\Tmax}
\end{equation}
for the $V'(t_{r},i)$, where $\Tmax \in \BigO(\tinit)$ is the maximal length of the string. Using a multiplicative drift theorem, Theorem~\ref{thm:multiplicativeDrift}, after $\BigO(\tinit \log \tinit)$ rounds we have $V'(t_{r},i) = 0$ with very high probability. By a union bound over all $i$, with high probability there is no $i$ left after $\BigO(\tinit \log \tinit)$ rounds for which $V'(t_{r},i)<0$. This proves the theorem modulo the statement on the bloat.

Regarding the bloat, we note that in expectation the offspring has the same size as the parent and the size of the tree does not change significantly by such unbiased fluctuations. However, in some situations bigger offsprings are more likely to be accepted or shorter offsprings are more likely to be rejected. This results in a positive drift for the size, which we need to bound. Note that the biased drift is caused purely by the selection process. We will show that offsprings are rarely rejected and bound the drift of $s(t_r)$ by (essentially) the probability that the offspring is rejected.


Similar as before, for an expressed variable $\var_i$ we let $V(t_r,i) \coloneq \#\{\var_i\text{-literals}\}-\#\{\nonvar_i\text{-literals}\}\geq 0$. An important insight is that the offspring can only be rejected if there is some expressed $\var_i$ such that at least $V(t_r,i)+1$ mutations touch $i$, i.e., they touch $\var_i$-literals or $\nonvar_i$-literals.\footnote{Some borders cases are neglected in this statement.} We want to show that this does not happen frequently. The probability to touch $\var_i$-literals or $\nonvar_i$-literals at least $k$ times falls geometrically in $k$, as we show in Lemma~\ref{lem:touched}. So for this outline we will restrict to the most dominant case $V(t_r,i)=0$.

Assume that we are in a situation where the size of the tree has grown at most by a constant factor. Similar as before, we may bound the drift of $V(t_r,i)$ in rounds that touch $i$ by 
\begin{equation}\label{outlineeq:drift}
\EE[V(t_{r},i)- V(t_{r+1},i) \mid V(t_r,i) = v, \text{$i$ touched in round $r$}] \leq \frac{Cvn}{\tinit}
\end{equation}
for a suitable constant $C>0$. The factor $n$ appears because we condition on $i$ being touched in round $r$, which happens with probability $\Omega(1/n)$.

Equation~\eqref{outlineeq:drift} tells us that the drift may be positive, but that it is relatively weak. In particular, for $v \leq N\coloneq\sqrt{\tinit/n}$, the drift is at most $\BigO(1/N)$. We prove that under such circumstances the expected return time to $0$ is large. More precisely, it can be shown with martingale theory (Theorem~\ref{lem:weakdrift}) that the expected number of rounds that touch $i$ to reach $V(t_r,i)=0$ from any starting configuration is at least $\Omega(N)$.\footnote{Interestingly, we also show that a substantial part of this expectation comes from return times of size $\Omega(N^2)$, which will be important to obtain tail bounds later on.} In particular, after $V(t_r,i)$ becomes positive for the first time, it needs in expectation $\Omega(N)$ rounds that touch $i$ to return to $0$. On the other hand, it only needs $\BigO(1)$ rounds that touch $i$ to leave $0$ again. Hence, $V(t_r,i)$ is only at $0$ in an expected $\BigO(1/N)$-fraction of all rounds that touch $i$.\footnote{This statement is more subtle than it may seem, and it is only true because the return times have the right tail distribution.} Thus the drift of $s(t_r)$ is also $\BigO(1/N)$.

In particular, if $\tinit \geq n\log^2 n$ then in $r_0 \in \BigO(\tinit \log \tinit)$ rounds the drift increases the size of the \gptree in expectation by at most $r_0/N \in \BigO(\tinit)$. Hence, we expect the size to grow by at most a constant factor. In fact, we provide strong tail bounds showing that it is rather unlikely to grow by more than a constant factor. The exact statement can be found in Theorem~\ref{thm:lengthMA}.

\subsubsection{Preparations}

We now turn to the formal proof of Theorem~\ref{thm:majority}.

\paragraph{Notation.} We start with some notation and technical lemmas. For a variable $i\in [n]$, we say that $i$ is \emph{touched} by some mutation, if the mutation inserts, delete or changes a $\var_i$ or $\nonvar_i$ variable, or if it changes a variable into $\var_i$ or $\nonvar_i$. We say that a mutation touches $i$ \emph{twice} if it relabels a $\var_i$-literal into $\nonvar_i$ or vice versa. Note that a relabeling operation has only probability $\BigO(1/n)$ to touch a literal twice. We call a round an \emph{$i$-round} if at least one of the mutations in this round touches $i$. Finally, we say that $i$ is \emph{touched $s$ times} in a round if it is touched exactly $s$ times by the mutations of this round (counted with multiplicity $2$ for mutations that touch $i$ twice). 

For a GP-tree $t$, let
$$
V(t,i) \coloneq 
\begin{cases}
-1,				&\mbox{no $\var_i$ or $\nonvar_i$ appear in the tree;}\\
-z,				&\mbox{there are $z>0$ more $\nonvar_i$ than $\var_i$;}\\
z,				&\mbox{$\var_i$ is expressed, and there are $z \geq 0$ more $\var_i$ than $\nonvar_i$.}
\end{cases}
$$
In particular, $i$ is expressed if and only if $V(t,i) \geq 0$. Note that $V(t,i)=-1$ may occur either if $\var_i$ and $\nonvar_i$ do not appear at all, or if exactly one more $\nonvar_i$ than $\var_i$ appears. Both cases have in common that $i$ will be expressed after a single insertion of $\var_i$. 

Note that a mutation that touches $i$ once can change $V(t,i)$ by at most $1$, with one exception: if $V(t,i) =1$ and there is only a single positive $\var_i$-literal, then $V(t,i)$ may drop to $-1$ by deleting this literal. Conversely, $V(t,i)$ can jump from $-1$ to $1$ by the inverse operation. In general, if $i$ is touched at most $s$ times and $V(t,i) >s$ then $V(t,i)$ can change at most by $s$; it can change sign only if $|V(t,i)| \leq s$. We say that a variable $i$ is \emph{critical} in a round if $V(t,i) \geq 0$, and $i$ is touched at least $V(t,i)$ times in this round; we call the variable \emph{non-critical} otherwise. Moreover, we say that a variable is \emph{positive critical} if it is critical and $V(t,i)$ is strictly positive. We say that a round is (positive) critical if there is at least one (positive) critical variable in this round. Note that in a non-critical round, the fitness of the \gptree cannot decrease. 

\paragraph{Many Mutations.}
We conclude our preparations with a lemma stating that it is exponentially unlikely to have many mutations, even if we condition on some variable to be touched.
\begin{lemma}\label{lem:touched}
There are constants $C,\delta >0$ and $n_0 \in \N$ such that the following is true for every $n\geq n_0$, every \gptree $t$ with $T \geq 2n$ leaves
, and every $\kappa\geq 2$. Let $i \in [n]$, and let $k$ denote the number of mutations in the next round. Then:
\begin{enumerate}
\item \label{lem:touched_1} $\Pr[k \geq \kappa] \leq e^{-\delta \kappa}$.
\item \label{lem:touched_2}$\Pr[k = 1 \mid i \text{ touched}] \geq \delta$.
\item \label{lem:touched_3}$\Pr[k \geq \kappa \mid i \text{ touched}] \leq e^{-\delta \kappa}$.
\item \label{lem:touched_4}$\EE[k \mid i \text{ touched}] \leq C$.
\end{enumerate}
\end{lemma}
\begin{proof}
Note that all statements are trivial if the \oneonegp uses $k=1$ deterministically. So for the rest of the proof we will assume that $k$ is $1+\Pois(1)$-distributed. We will use the well known inequality
\begin{align}\label{eq:Poissonbound}
\Pr[\Pois(\lambda) \geq x] \leq e^{-\lambda}\left(\frac{e\lambda}{x}\right)^x
\end{align}
for the Poisson distribution~\cite{BookMitUp}. In our case ($\lambda =1$, $x=\kappa-1$), and using $e^{-1} \leq 1$, we can simplify to
\begin{equation}\label{eq:Poisson}
\Pr[\Pois(1) \geq \kappa-1] \leq \left(\frac{e}{\kappa-1}\right)^{\kappa-1}.
\end{equation}

\noindent\ref{lem:touched_1}: First consider $\kappa \geq 4$. Then, using $\kappa -1 \geq \kappa/2$ we get from \eqref{eq:Poisson}:
\begin{align*}
\Pr[k \geq \kappa] & =  \Pr[\Pois(1) \geq \kappa-1]  \leq (e/3)^{\kappa/2} = e^{\log(e/3)\kappa/2}.
\end{align*}
Thus \ref{lem:touched_1} is satisfied for $\kappa \geq 4$ with $\delta \coloneq \log(e/3)/2$. By making $\delta$ smaller if necessary, we can ensure that \ref{lem:touched_1} is also satisfied for $\kappa \in \{2,3\}$., which proves this property.

\noindent\ref{lem:touched_2} and \ref{lem:touched_3}: Let $T = s(t)$ be the size of $t$ (the number of leaves). Additionally, we define the parameter
\[
x \coloneq \max\left\{\frac{\#\{\text{i-literals in $t$}\}}{T}, \frac{1}{n}\right\}.
\]
Note that the next mutation has probability at most $2x$ to touch $i$. Unfortunately, that is not true for subsequent mutations in the same round, which makes the proof considerably more complicated. 
We claim
\begin{align}\label{eq:k1lower}
\Pr[k=1 \text{ and $i$ touched}] & \geq \frac{x}{3e}.
\end{align}
To see the claim, first note that $\Pr[k=1] = 1/e$ by definition of the Poisson distribution. First, consider the case that $x = 1/n$. Then we have $\Pr[k=1 \text{ and $\var_i$ or $\nonvar_i$ inserted}] = 1/(3en)$, which implies~\eqref{eq:k1lower}. In the other case, the probability that a deletion operation picks a $\var_i$ or $\nonvar_i$ is $x$, so $\Pr[k=1 \text{ and $\var_i$ or $\nonvar_i$ inserted}] = x/(3e)$, which also implies~\eqref{eq:k1lower}. This proves~\eqref{eq:k1lower} in all cases.

We first prove the simpler case of large $x$; more precisely, let $x \geq 1/4$. With probability $1/e$ there is only one mutation and with probability at least $x/3 \geq 1/12$ this mutation deletes a $\var_i$ or $\nonvar_i$-literal. Hence,
\[
\Pr[k =1 \text{ and $i$ touched}] \geq \frac{1}{12}.
\]
This already implies \ref{lem:touched_2}, because 
\[
\Pr[k=1 \mid i \text{ touched}] \geq \Pr[k =1 \text{ and $i$ touched}] \geq \frac{1}{12e}.
\] 
Regarding \ref{lem:touched_3} it suffices to observe that
\begin{align}\label{eq:largex}
\Pr[k\geq \kappa \mid i \text{ touched}] & = \frac{\Pr[k \geq \kappa \text{ and $i$ touched}]}{\Pr[i \text{ touched}]} \nonumber \\
& \leq \frac{\Pr[k \geq \kappa]}{\Pr[k =1 \text{ and $i$ touched}]} \stackrel{1.}{\leq} 12e\cdot e^{-\delta \kappa},
\end{align}
which implies \ref{lem:touched_3} by absorbing the factor $12e$ into the exponential.


The case for smaller $x$ basically runs along the same lines, but will be much more involved. In particular, in~\eqref{eq:largex} we cannot use the trivial bounds in the second line. So assume from now on $x < 1/4$ and thus at most one fourth of the literals in $t$ are $i$-literals. In the following we will bound the probability to have $k > 1$ mutations such that at least one of them touches $i$. The probability to have $k=\kappa$ mutations is $\Pr[\Pois(1) = \kappa-1]$. We will first assume $k \leq 1/x$. Note for later reference that $k \leq 1/x \leq n \leq T/2$ in this situation.

We fix some value $k \leq 1/x$. Let us refer to the mutations by $M_1,\ldots,M_k$ and let $\kappa_i \coloneq \min\{1 \leq \kappa \leq k \mid M_\kappa \text{ touches }i\}$ be the index of the first mutation that touches $i$. If none of $M_1,\ldots,M_k$ touches $i$ then we set $\kappa_i \coloneq \infty$. We claim that for all $k\leq 1/x$ and all $1 \leq \kappa \leq k$,
\begin{align}\label{eq:kappai}
\Pr[\kappa_i \geq \kappa+1  \mid k, \kappa_i \geq \kappa] \geq 1 - 3x \geq e^{-6x},
\end{align}
where the last inequality holds since $x < 1/4$.

In order to see the the first inequality of~\eqref{eq:kappai} we distinguish two cases. If $x=1/n$, then the number of $i$-literals in $t$ is at most $Tx = T/n$. Since we condition on $\kappa_i \geq \kappa$, the number of $i$-literals is still at most $T/n$ after the first $\kappa-1$ operations. The number of leaves after $\kappa-1 < n$ operations is at least $T-n \geq T/2$. Hence, the probability to pick one of these leaves for deletion or relabeling is at most $(2/3) (T/n) / (T/2) < 2/n$. On the other hand, the probability to insert an $i$-literal or to relabel a leaf with $\var_i$ or $\nonvar_i$ is at most $1/n$. By the union bound, the probability to touch $i$ is at most $3/n$. This proves~\eqref{eq:kappai} if $x=1/n$.

The other case is very similar only involving different numbers. The number of $i$-literals in $t$ is $Tx$. Since $k \leq 1/x \leq T/2$, after $\kappa \leq k$ operations the size of the remaining tree is at least $T/2$. Therefore, the probability that $M_{\kappa}$ picks an $i$-literal for deletion or relabeling is at most $(2/3) xT /(T/2) \leq 2x$. On the other hand, the probability to insert an $i$-literal or to relabel a leaf with $\var_i$ or $\nonvar_i$ is at most $1/n \leq x$. By the union bound, the probability to touch $i$ is at most $3x$. This proves~\eqref{eq:kappai} if $x=\#\{i\text{-literals}\}/T$.

We can expand~\eqref{eq:kappai} to obtain the probability of $\kappa_i = \infty$. For $2 \leq k \leq 1/x$,
\begin{align*}
\Pr[\kappa_i = \infty \mid k] & = \prod_{i=1}^k \Pr[\kappa_i \geq \kappa+1  \mid k, \kappa_i \geq \kappa] \geq e^{-6kx},
\end{align*}
and consequently 
\begin{align*}
\Pr[i \text{ touched} \mid k] & = 1- \Pr[\kappa_i = \infty \mid k] \leq 1 - e^{-6kx} \leq 6kx.
\end{align*}
For $k > 1/x$ we will use the bound $\Pr[i \text{ touched} \mid k ] \leq 1$. To ease notation, we will assume in our formulas that $1/x$ is an integer. Then we may bound 
\begin{align*}
\hspace{-0em}\Pr[k \geq 2 \text{ and $i$ touched}] &\leq  \sum_{\kappa=2}^{1/x} \Pr[k= \kappa] \Pr[i \text{ touched} \mid k = \kappa] + \hspace{-.6em}\sum_{\kappa = 1+1/x}^{\infty} \hspace{-.6em}\Pr[k= \kappa] \\
& \stackrel{1.}{\leq} \sum_{\kappa=2}^{1/x} e^{-\delta \kappa}6\kappa x  +  \hspace{-.6em}\sum_{\kappa = 1+1/x}^{\infty} \hspace{-.6em} e^{-\delta \kappa}
\leq x \sum_{\kappa=2}^\infty (6 \kappa +\tfrac{1}{x}e^{-\delta/x})e^{-\delta \kappa} \\ 
&\leq Cx
\end{align*}
for a suitable constant $C>0$, since the function $\tfrac{1}{x}e^{-\delta/x}$ is upper bounded by a constant for $x \in (0,1]$. Together with~\eqref{eq:k1lower}, we get
\begin{align*}
\frac{1}{\Pr[k=1 \mid i \text{ touched}]} &= 1 + \frac{\Pr[k\geq 2 \text{ and $i$ touched}]}{\Pr[k=1 \text{ and $i$ touched}]}\\
& \leq 1+ \frac{Cx}{x/(3e)} = 1+3eC.
\end{align*}
This proves \ref{lem:touched_2} for $\delta \coloneq 1/(1+3Ce)$. For \ref{lem:touched_3} we compute similar as before
\begin{align*}
\Pr[k \geq \kappa \text{ and $i$ touched}] & \leq  \sum_{\kappa'=\kappa}^{1/x} \Pr[k= \kappa'] \Pr[i \text{ touched} \mid k = \kappa'] + \hspace{-1.5em}\sum_{\kappa' = \max\{\kappa,1+1/x\}}^{\infty} \hspace{-2em}\Pr[k= \kappa'] \\
& \leq \sum_{\kappa'=\kappa}^{1/x} e^{-\delta \kappa'}6\kappa'x  +  \sum_{\kappa' = \max\{\kappa,1+1/x\}}^{\infty} e^{-\delta \kappa'}\\
&\leq xe^{-\delta \kappa/2} \sum_{\kappa'=1}^\infty (6 \kappa'+ \tfrac{1}{x}e^{-\delta/x})e^{-\delta \kappa'/2}  \leq Cxe^{-\delta\kappa/2}
\end{align*}
for a suitable constant $C>0$. Therefore, as before, 
\begin{align*}
\frac{1}{\Pr[k\geq \kappa \mid i \text{ touched}]} &= 1 + \frac{\Pr[k <  \kappa \text{ and $i$ touched}]}{\Pr[k\geq \kappa \text{ and $i$ touched}]} \geq 1 + \frac{\Pr[k =1 \text{ and $i$ touched}]}{\Pr[k\geq \kappa \text{ and $i$ touched}]} \\
& \geq 1+ \frac{x/(3e)}{Cxe^{-\delta\kappa/2}} \geq \frac{1}{3eC}e^{\delta\kappa/2}.
\end{align*}
This proves \ref{lem:touched_3}, since we may decrease $\delta$ in order to swallow the constant factor $3eC$ by the term $e^{\delta\kappa/2}$.

\noindent\ref{lem:touched_4}: This follows immediately from \ref{lem:touched_3}, because
\begin{align*}
\EE[k \mid i \text{ touched}] & = \sum_{\kappa \geq 1} \Pr[k \geq \kappa \mid i \text{ touched}]  \leq 1 + \sum_{\kappa \geq 2} e^{-\delta \kappa},
\end{align*}
and the latter sum is bounded by an absolute constant.
\end{proof}

\subsubsection{Bloat Estimation}

The main part of the proof is to study how the size of the \gptree increases. We show that it increases by only a little more than a constant factor within roughly $\tinit \log\tinit$ rounds if $\tinit \in \omega(n\log^2 n)$.
\begin{theorem}\label{thm:lengthMA}
There is $\eps >0 $ such that the following holds. Let $f = f(n) \in \omega(1)$ be any growing function with $f(n) \in \LittleO(n)$. Let $\Tmin \coloneq \max\{\tinit, f(n) \, n\log^2n\}$. Then for sufficiently large $n$, with probability at least $1-\exp(-\eps\sqrt{f(n)})$, within the next $r_0\coloneq\eps f(n) \Tmin \log \Tmin$ rounds the tree has never more than $\Tmax \coloneq \sqrt{f(n)}\Tmin$ leaves. %
\end{theorem}\

The proof of Theorem~\ref{thm:lengthMA} is the most technical part of the proof and this whole subsection is devoted to it. First, we provide an outline of the basic ideas, adding some actual numbers to the general outline presented in Section~\ref{sec:upperboundoutline}. We will couple the size of the GP tree to a different process $S = (S_r)_{r\geq 0}$ on $\N$ which is easier to analyze. The key idea is that we only have a non-trivial drift in rounds in which the offspring is rejected. As we will see later, this event does not happen often. Formally, we define $S$ by a sum $S_r = \Tmin + \sum_{j=1}^r (X_j' +X_j)$, where $X_j'$ are independent random variables with zero drift, and $X_j$ are only non-zero in critical rounds. 

The most difficult part is to bound the contribution of the $X_j$, i.e, to show that most rounds are non-critical. To this end, we will show that the random variables $V(t,i)$, once they are non-negative, follow a random walk as described in Theorem~\ref{lem:weakdrift}, with parameter $N \coloneq \sqrt{\Tmin/n} \geq \sqrt{f(n)}\log \Tmin$. For the purpose of this outline we consider only rounds in which at most one variable $i\in [n]$ with $V(t,i) =0$ is critical. This (almost) covers the case when the number $k$ of mutations in a round is constantly one, but similar arguments transfer to the case when $k$ is $1+\Pois(1)$-distributed. Whenever $i$ is touched in such a round then $V(t,i)$ has probability $\Omega(1)$ to increase, so the state $V(t,i)=0$ will only persist for $\BigO(1)$ rounds that touch $i$. On the other hand, after being increased, it needs in expectation $\Omega(N)$ $i$-rounds to return to zero. Intuitively, this means that in a random $i$-round, the probability to encounter $V(t,i)=0$ is $\BigO(1/N)$. Note that this intuition is not quite correct, but we can use Lemma~\ref{lem:twostates} for the formal argument. Since each round touches only $\BigO(1)$ variables, and each of them has only probability $\BigO(1/N)$ to be critical, there are only $\BigO(r_0/N) \in \BigO(\eps\sqrt{f(n)}\Tmin)$ critical rounds within $r_0$ rounds. Thus the size of the \gptree grows only roughly by a constant factor in $\Tmin \log \Tmin$ rounds.

\begin{proof}[Proof of Theorem~\ref{thm:lengthMA}] 
We will prove the theorem under the assumption that the size of the \gptree never falls below $\Tmin$. This is justified because we can track the process until either $r_0$ rounds have passed or the size of the \gptree falls below $\Tmin$ in some round $r\leq r_0$. In the former case we are done, in the latter case we apply the same argument again starting in the next round in which the size of the \gptree exceeds $\Tmin$.\footnote{We are slightly cheating here, because for $k\sim1+\Pois(1)$, the size of the \gptree may jump to something strictly larger than $\Tmin$ in one step. However, our proof also works if we start with any \gptree of size at most $2\Tmin$, and the probability to increase the size of the \gptree by more than $\Tmin$ in one step is negligibly small.} 

Let $t$ be the \gptree in round $j$, let $k$ be the number of mutations in this round, and let $t'$ be the tree resulting from these mutation. We set $X_{j+1}' \coloneq s(t')|-s(t)$, and
\begin{equation}\label{eq:defofX_r}
X_{j+1} \coloneq 
\begin{cases}
k,			&\mbox{if round $j$ is positive critical;}\\
0,			&\mbox{otherwise.}
\end{cases}
\end{equation}
As mentioned in the outline, we define $S_r \coloneq \Tmin + \sum_{j=1}^r (X_j' +X_j)$. We first show that the size of the \gptree after $r$ rounds is at most $S_r$.

The fitness of $t'$ can only be smaller than the fitness of $t$ if there is at least one index $i$ for which $V(t,i)$ changes from non-negative to negative, which can only happen in positive critical rounds. In particular, in the second case of~\eqref{eq:defofX_r} we have $f(t') \geq f(t)$, and hence the \gptree $t'$ is accepted. Thus, in this case we have $S_{r+1}-S_r = X_{r+1}' + X_{r+1} = s(t')-s(t)$, so $S_j$ and the size of the \gptree both change by the same amount. For the first case of~\eqref{eq:defofX_r}, we have $S_{r+1}-S_r = k + s(t')-s(t) \geq \max\{0, s(t')-s(t)\}$. Since the size of the \gptree changes either by $s(t')-s(t)$ (if $t'$ is accepted) or by $0$ (if $t'$ is rejected), the increase of $S_r$ is at least the increase of the size of the \gptree. Since this is true for all cases, the size of the \gptree is at most $S_r$, as claimed. We will derive upper bounds on $S_r$ in the following.

In order to bound $S_r = \sum_{j=1}^r (X_j+X_j')$ we will prove separately that each of the bounds $\sum_{j=1}^r X_j' \leq \Tmax/3$ and $\sum_{j=1}^r X_j \leq \Tmax/3$ holds with probability at least $1-\exp\{-\Omega(\sqrt{f(n)})\}$. By the union bound, it will follow that \emph{both} bounds together hold with probability at least  $1-\exp\{-\Omega(\sqrt{f(n)})\}$. The two bounds will imply that the size of the \gptree is at most $\Tmin + 2 \Tmax/3 \leq \Tmax$, thus proving the theorem. Recall that we need to consider the range $1\leq r \leq r_0 = f(n) \eps\Tmin \log \Tmin$. 

First we bound $X_j'$.
 
\noindent For $\sum_{j=1}^r X_j'$, note that each $X_j'$ is the sum of $k$ Bernoulli-type random variables (with values $+1$ for insertion, $-1$ for deletion, and $0$ for relabeling), where $k$ is either constantly $1$ or $1+\Pois(1)$-distributed, depending on the algorithm. Let us denote by $K_r$ the total number of Bernoulli-type variables (i.e., the total number of mutations in $r$ rounds). In the case where we always choose $k=1$, we have trivially $K_r=r$. In the case $k \sim 1+\Pois(1)$ we have $K_r \sim r+\Pois(r)$ since the sum of independent Poisson distributed random variables is again Poisson distributed. Since $\Pois(r)$ is dominated by $\Pois(r_0)$, we have
\[
\Pr[K_r \geq 3 r_0] \leq \Pr[\Pois(r_0) \geq 2r_0] \stackrel{\eqref{eq:Poissonbound}}{\leq} \frac{e^{-r_0}(er_0)^{2r_0}}{(2r_0)^{2r_0}} = \left(\frac{e}{4}\right)^{r_0}
\]
for each $r \leq r_0$. Note that this estimate holds also for the case that all $k$ are one, because then the probability on the left is zero. Taking a union bound over all $1\leq r \leq r_0$ we see that with exponentially high probability\footnote{that means with probability $1-e^{-\Omega(r_0)}$.} $K_r \leq 3r_0$ also holds uniformly for all $1\leq r \leq r_0$. For each mutation the probability of insertion, deletion, and substitution is $1/3$ each, i.e., each of the $K_r$ Bernoulli-type random variables contributes $+1$, $-1$, or $0$, with probability $1/3$ each. 
Thus we may use the Chernoff bound, Theorem~\ref{thm:Chernoff}, to infer that with sufficiently high probability $\sum_{j=1}^r X_j' \leq r_0^{3/4} < \Tmax/3$ holds uniformly for all $1\leq r \leq r_0$. In particular, this probability is $1-\exp\large\{-\Omega(\sqrt{f(n)})\large\}$.



It remains to bound $\sum_{j=1}^r X_j$. Recall that $X_j$ is either zero or the the number of mutations applied in the $j$-th round. Therefore, the sum is non-decreasing in $r$ and it suffices to bound the sum for $r=r_0$. And the same bound will follow for all $r \leq r_0$. 

We fix some $i \in [n]$ and consider the random walk of the variable $V(t_r,i)$. Recall that we assume the size of the \gptree $t_r$ to be at least $\Tmin$. Since $V(t_r,i)$ can only change in $i$-rounds, it makes sense to study the random walk by only considering $i$-rounds. We will apply Theorem~\ref{lem:weakdrift} with $N \coloneq \sqrt{\Tmin/n}$ to this random walk. To this end, in the following paragraphs we prove that the random walk that $V(t_r,i)$ performs in $i$-rounds satisfies the conditions of Theorem~\ref{lem:weakdrift}.


Now we are ready to compute the drift of $X_j$.

\noindent Let us first consider $v\geq 1$, and compute the drift 
\[
\Delta_{v,i} \coloneq \EE[V(t_{r+1},i)- V(t_r,i) \mid V(t_r,i) = v, r \text{ is $i$-round}].
\] 
We mind the reader to not confuse this drift with the drift of $S_r$, which is a very different concept. The notation $\Delta_{v,i}$ is slightly abusive because the drift does depend on $t_r$ too. However, we will derive lower bounds on the drift which are independent of $t_r$, thus justifying the abuse of notation. 
In fact, we will compute the drift of
\[
\Delta_{v,i}' \coloneq \EE[V(t_{r}',i)- V(t_r,i) \mid V(t_r,i) = v, r \text{ is $i$-round}],
\] 
where $t_r'$ is the offspring of $t_r$. In other words, we ignore whether the offspring is accepted or not. Note that this can only decrease the drift, since a mutation that causes $t_r'$ to be rejected can not increase $V(t_r,i)$. Hence, any lower bound on $\Delta_{v,i}'$ is also a lower bound on $\Delta_{v,i}$.

Let $\event{E}_r$ be the event that $r$ is an $i$-round. Note that
\begin{equation}\label{eq:E_r}
\Pr[\event{E}_r] \in \Omega(1/n), 
\end{equation}
since we always have probability $1/(3n)$ to touch $i$ with an insertion.

Consider any round $r$ conditioned on $\event{E}_r$ and let $M$ be a mutation in round $r$. If $M$ does not touch $i$, then $M$ does not change $V(t_r,i)$ and the contribution to the drift is zero. Next we consider the case that $M$ is an insertion of either $\var_i$ or $\nonvar_i$. Both cases are equally likely and the case that $M$ is an insertion contributes zero to the drift. By the same argument, the cases that $M$ relabels a non-$i$-literal into $\var_i$ or into $\nonvar_i$ cancel out and together contribute zero to the drift. 

Next consider deletions of $\var_i$ or $\nonvar_i$. This case is not symmetric, since there are $v \geq 1$ more $\var_i$ than $\nonvar_i$. Assume that the number of $\var_i$ is $x+v$, while the number of $\nonvar_i$ is $x$, for some $x\geq 0$. Consider the first $x$ occurrences of $\var_i$. Then the probability that a deletion $M$ picks one of these first $\var_i$ equals the probability that $M$ picks one of the $\nonvar_i$. As before, both cases are equally likely. Therefore, the contribution to the drift from either picking one of the first $x$ occurrences of $\var_i$ or any occurrence of $\nonvar_i$, cancel out. For the remaining $v$ literals $\var_i$ the unconditional probability that a deletion picks one of them is $v/|t_r| \leq v/\Tmin$, where $|t_r| \geq \Tmin$ is the current size of the \gptree. Thus the conditional probability (on $\event{E}_r$) to pick one of them is at most $\BigO(vn/\Tmin)$ by~\eqref{eq:E_r}. Since the conditional expected number of deletions is $\EE[\#\text{ deletions} \mid \event{E}_r] \in\BigO(1)$ by~Lemma~\ref{lem:touched}, the deletions contribute $- \BigO(vn/\Tmin)$ to the drift $\Delta_{v,i}$. By the same argument we also get a contribution of $- \BigO(vn/\Tmin)$ for relabelings of $\var_i$-literals or $\nonvar_i$-literals. 

Summarizing, the only cases contributing to $\Delta_{v,i}'$ are deletions and relabeling of $i$-literals, and they contribute not less than $-\BigO(vn/\Tmin)$, which is $-\BigO(\sqrt{n/\Tmin})$ for $v \leq N = \sqrt{\Tmin/n}$. All other cases contribute zero to $\Delta_{v,i}'$. Therefore, the random walk of $V(t_r,i)$ (where we only consider rounds which touch $i$) satisfies the first condition of Theorem~\ref{lem:weakdrift} with $N = \sqrt{\Tmin/n}$.


Now we consider the step size and the initial increase of $X_j$.

The second condition (small steps) of Theorem~\ref{lem:weakdrift} follows from Lemma~\ref{lem:touched}. Finally, for the third condition (initial increase) we show that for every $v \leq N$, where $N = \sqrt{\Tmin/n}$ and every $n$ sufficiently large, with probability at least $\delta$ the next non-stationary step increases $V(t_r,i)$ by exactly one. Note that by Lemma~\ref{lem:touched}, an $i$-round has probability $\Omega(1)$ to have exactly one mutation. Now we distinguish two cases: if there are less than $s(t_r)/n$ occurrences of $\var_i$ then the probability to touch $i$ in any way is $\BigO(1/n)$ and the probability of inserting an $\var_i$-literal is $\Omega(1/n)$. Hence, conditioned on touching $i$, with probability $\Omega(1)$ the only mutation in this round is an insertion of $\var_i$. 

For the other case, assume there are more than $s(t_r)/n \geq \Tmin/n \in \omega(1)$ occurrences of $i$-literals. Additionally, assume that $v \leq \sqrt{\Tmin/n} < (1/3) s(t_r)/n$, where the last inequality holds for $n$ large enough since then $\Tmin/n$ is large enough. Then $\nonvar_i$ occurs at least half as often as $\var_i$, and thus the probability of deleting or relabeling a $\nonvar_i$-literal is at least half as big as the probability to delete or relabel an $\var_i$-literal. Therefore, a mutation that touches $i$ is with probability $\Omega(1)$ a deletion of $\nonvar_i$. So in both cases the first mutation that touches $i$ increases $V(t_r,i)$ with probability $\Omega(1)$. This proves that the third condition of Theorem~\ref{lem:weakdrift} is satisfied. 


We can now put everything together regarding the behavior of $X_j$.

\noindent So far, we have shown that $V(t_r,i)$ performs a random walk that satisfies the conditions of Theorem~\ref{lem:weakdrift}. Hence, for $0 < v < \eps' N = \eps'\sqrt{\Tmin/n}$ the expected hitting time of $\{[0,1,\ldots,v]\}$ when starting at any value larger than $v$ is $\Omega(\sqrt{\Tmin/n})$, for a suitable constant $\eps'>0$. Moreover, with probability $\Omega(1/N)$ the hitting time is at least $\Omega(N^2)$.

Now we have all ingredients to bound the expected number of positive critical rounds. We fix a variable $i$ and some $v \geq 0$ and aim to bound the number of rounds, in which $V(t_r,i) = v$ and $i$ is a critical variable. For $v \geq \eps' N \geq \eps'\sqrt{f(n)} \, \log \Tmin $, with probability at least $1-e^{-\Omega(N)} \geq 1-\exp\{-\Omega(\sqrt{f(n)})\}/\Tmin$ this does not happen in a specific round by Lemma~\ref{lem:touched}. By a union bound, with probability $1-\exp\{-\Omega(\sqrt{f(n)})\}$ it never happens for any variable $i$ and any of $r_0$ rounds, with room to spare. So we may assume $0 \leq v < \eps N$. We use Lemma~\ref{lem:twostates} to estimate how many $i$-rounds occur with $V(t_r,i) = v$ before for the first time $V(t_r,i) > v$. For this purpose we check the conditions of Lemma~\ref{lem:twostates}. In each $i$-round with $V(t_r,i) = v$, with probability $\Omega(1)$ the value of $V(t_r,i) = v$ increases strictly by Lemma~\ref{lem:touched}. On the other hand, once $V(t_r,i) > v$ it takes in expectation at least $\Omega(\sqrt{\Tmin/n})$ $i$-rounds before the interval $[0,1,\ldots,v]$ is hit again, and it takes at least $\Omega(\Tmin/n)$ $i$-rounds with probability at least $\Omega(\sqrt{n/\Tmin})$. Thus we are in the situation of Lemma~\ref{lem:twostates} with $\delta \in \Omega(1)$ and $s= \Theta(\sqrt{\Tmin/n})$.

Let $E_i$ denote the number of $i$-rounds and let $E_{i,v}$ be the number of $i$-rounds with $V(t_r,i) = v$. Note that we can only apply Lemma~\ref{lem:twostates} if $E_i \geq s$. However, in each round we have probability at least $1/(3n)$ to insert an $i$-literal. Hence, $\EE[E_i] \geq r_0/(3n) \in \Omega(f(n)\log n)$. In particular, by the Chernoff bound, Theorem~\ref{thm:Chernoff}, $\Pr[E_i < r_0/(6n)] \leq e^{-\Omega(f(n) \log n)} \ll (1/n) e^{-\Omega(f(n))}$. Hence, after a union bound over all $i$, we observe that with probability $1-e^{-\Omega(f(n))}$ we have $E_i \geq r_0/(6n)$ for all $1\leq i \leq n$, and we will assume this henceforth. In particular, $E_i \geq r_0/(6n) \geq s$. Thus we may apply Lemma~\ref{lem:twostates} with $r= E_i$ and obtain
\begin{align*}
\EE[E_{i,v}] \leq C  \sqrt{\frac{n}{\Tmin}} \EE[E_i]
\end{align*}
for a suitable constant $C>0$. Moreover, by the tail bound in Lemma~\ref{lem:twostates}, 
\begin{align}\label{eq:tailboundEi}
\Pr\left[E_{i,v} \leq 2C \sqrt{\frac{n}{\Tmin}} E_i \right] & \geq 1- e^{-r_0/(12ns)} \in 1- e^{-\Omega(\sqrt{f(n)}\log \Tmin)} \nonumber \\
& \geq 1- \frac{1}{nN}e^{-\Omega(\sqrt{f(n)})}.
\end{align}
By a union bound over all $i$ and $v$ we see that with probability $1-\exp\{-\Omega(\sqrt{f(n)})\}$ the bound $E_{i,v} \leq 2C \sqrt{n/\Tmin} E_i$ from~\eqref{eq:tailboundEi} holds for all $1\leq i\leq n$ and all $1\leq v \leq \sqrt{N}$. So again we may assume this from now on.

An $i$-round with $V(t_r,i) = v$ has probability $e^{-\Omega(v)}$ for $i$ to be critical by Lemma~\ref{lem:touched}. Therefore, the expected number of critical rounds within the first $r_0$ rounds is at most
\begin{align}\label{eq:expcritical}
\EE[\#\{\text{critical rounds}\}] & \leq \sum_{\substack{i \in [n] \\0 \leq v \leq \eps N}} e^{-\Omega(v)}\cdot \EE[E_{i,v}] \in \BigO\left(\sqrt{\frac{n}{\Tmin}}\right) \sum_{i \in [n]}\EE[E_i].
\end{align}
The bound $e^{-\Omega(v)}$ that an $i$-round with $V(t_r,i) = v$ is critical holds independently of all previous rounds. Therefore, as before we can use the Chernoff bound to amend~\eqref{eq:expcritical} by the corresponding tail bound and obtain with probability at least $1- e^{-\Omega(\sqrt{f(n)})}$ that
\begin{align}\label{eq:tailboundcritical}
\#\{\text{critical rounds}\} \leq C'\sqrt{\frac{n}{\Tmin}} \sum_{i \in [n]}E_i
\end{align}
for a suitable constant $C'>0$.

We bound the sum further by observing that in each round only $\BigO(1)$ literals are touched in expectation and the number of touched literal drops at least exponentially. Therefore, $\sum_{i \in [n]}\EE[E_i] \in \BigO(r_0)$ and by standard concentration bounds~\cite[Theorem11]{kotzing2016concentration} with probability $1- \exp\{-\Omega(\sqrt{f(n)})\}$ the expectation is not exceeded by more than a constant factor. Moreover, by assumption we have $\Tmin \geq f(n)n\log^2 n$, which implies $\Tmin \geq (1/2) f(n) n \log^2 \Tmin$ for sufficiently large $n$. 
Hence, with probability $1- \exp\{-\Omega(\sqrt{f(n)})\}$
\begin{align*}
\#\{\text{critical rounds}\} & \in \BigO\left(r_0\sqrt{\frac{n}{\Tmin}}\right) \in \BigO\left(\frac{r_0}{\sqrt{f(n)}\log\Tmin}\right)\\
&\leq \tfrac{1}{12} \sqrt{f(n)}\Tmin,
\end{align*}
where the last step follows from $r_0 = f(n)\eps\Tmin \log \Tmin$ if $\eps>0$ is sufficiently small. 
Since $X_j$ is zero in non-critical rounds and is bounded by $1+\Pois(1)$ in critical rounds, as before we may use~\cite[Theorem11]{kotzing2016concentration} to get the following tail bound.
\[
\Pr\left[\sum_{j=1}^{r_0} X_j \leq \tfrac13 \sqrt{f(n)}\Tmin\right] \in 1- e^{-\Omega(\sqrt{f(n)})}.
\]
Thus we have shown that with sufficiently large probability $\sum_{j=1}^{r_0} X_j \leq \tfrac{1}{3} \sqrt{f(n)}\Tmin = \Tmax /3$. This proves the desired bound on $S_r$ and thus concludes the proof of Theorem~\ref{thm:lengthMA}.
\end{proof}

\subsubsection{Run Time Bound}


For technical reasons, we first need to prove a rather technical statement that holds with high probability. 

\begin{lemma}\label{lem:MAwhp}
There is $\eps >0$ such that the following holds for any growing function $f(n) \in \omega(1)$ with $f(n) \in \LittleO(n)$. Let $\Tmin \coloneq \max\{\tinit, f(n) n\log^2 n\}$. If $n$ is sufficiently large, then for any starting tree, with probability at least $1-\exp\large\{-f(n)^{1/4}\large\}$ the \oneonegp without bloat control on \majority finds a global optimum within $r_0 \coloneq \eps f(n) \Tmin \log \Tmin$ rounds, and the size of the \gptree never exceeds $\Tmax = \sqrt{f(n)}\Tmin$.
\end{lemma}

\begin{proof}
We already know by Theorem~\ref{thm:lengthMA} that with probability $1-\exp\{-\Omega(\sqrt{f(n)})\}$ the size of the \gptree does not exceed $\Tmax$ within $r_0$ rounds. We fix a variable $i$, which is not expressed at the beginning, and consider $V'(t_r,i) \coloneq \max\{-V(t_r,i),0\}$. We claim that $V'(t_r,i)$ has a multiplicative drift,
\begin{equation}\label{eq:MA}
\EE[V'(t_{r},i)-V'(t_{r+1},i) \mid V'(t_{r},i) = v] \geq \frac{v}{3e\Tmax},
\end{equation}
for all $v\geq 0$, as long as $i$ is not expressed.
In order to prove~\eqref{eq:MA} we first consider insertions. It is equally likely to insert $\var_i$ (which decreases $V'(t_{r},i)$) and $\nonvar_i$ (which increases $V'(t_{r},i)$). Moreover, whenever the offspring is accepted after inserting $\nonvar_i$, it is also accepted after inserting $\var_i$. Therefore, the contribution to the drift from insertions is at least zero. Analogously, relabeling a non-$i$-literal into an $i$-literal contributes at least zero to the drift. For deletions, with probability at least $1/(3e)$ we have exactly one mutation, and this mutation is a deletion. In this case, the probability to delete a $\nonvar_i$-literal is exactly by $v/s(t_r) \geq v/\Tmax$ larger than the probability to delete an $\var_i$-literal. Since we always accept deleting a single $\nonvar_i$-literal, this case contributes no less than $-{v}/{(3e\Tmax)}$ to the drift. For all the other cases (several deletions, relabeling of one or several $i$-literals), it is always more likely to pick a $\nonvar_i$-literal for deletion/relabeling than a $\var_i$-literal and it is more likely to accept the offspring if a $\nonvar_i$-literal is deleted/relabeled. Therefore, these remaining cases contribute at least zero to the drift. This proves~\eqref{eq:MA}.

We next show that for $V(t_r,i)=0$ in the next $i$-round with probability $\Omega(1)$ the literal $x_i$ is expressed in the offspring and no other literal becomes unexpressed.  We call such a round \emph{$i$-fixing}. Note that the number of expressed literals can never decrease, so $x_i$ can only become unexpressed if a literal $x_j$ becomes expressed in the same round. In this case we can just swap the roles of $i$ and $j$ for the remainder of the argument. So we may assume that after an $i$-fixing round the literal $x_i$ stays expressed forever. Then it suffices to show that for every $i$, if $i$ is not expressed for a sufficient number of rounds, then there is an $i$-fixing round.

Note that a sufficient condition for an $i$-fixing round is that there is only a single mutation which inserts a new $\var_i$-literal or deletes a $\nonvar_i$-literal. The probability to insert a new $\var_i$-literal equals the probability to insert a new $\nonvar_i$-literal, to create a $\var_i$-literal by relabeling or to create a $\nonvar_i$-literal by relabeling. On the other hand, the probability to delete a $\nonvar_i$-literal equals the probability to delete a $\var_i$-literal (since $V(t_{r},i)=0$), to relabel an $\var_i$-literal and to relabel a $\nonvar_i$-literal. Thus, the probability that an $i$-round with only a single mutation is $i$-fixing is at least $1/3$. Moreover, an $i$-round has probability $\Omega(1)$ to consist of a single mutation by Lemma~\ref{lem:touched}. This proves that for $V(t_r,i)=0$ the next $i$-round has probability $\Omega(1)$ to be $i$-fixing.

By the Multiplicative Drift Theorem
~\ref{thm:multiplicativeDrift}, $V'(t_{r},i)$ reaches $0$ after at most $r_{\text{init}} \coloneq 3e \Tmax (k+\log \Tmax)$ steps with probability at least $1-e^{-k}$, for a parameter $k>0$ that we fix later. Moreover, once at $0$ the next $i$-round is $i$-fixing with probability $\Omega(1)$. If it is not $i$-fixing, then $V'(t_{r},i)$ may jumps from $0$ to a positive value. This value will be at most $k$ with probability at least $1-e^{-\Omega(k)}$ by Lemma~\ref{lem:touched}, and again by the Multiplicative Drift Theorem $V'(t_{r},i)$ will return to $0$ after $r_{\text{return}}\coloneq3e \Tmax (k+\log \log k+\BigO(1)))$ steps with probability at least $1-e^{-\Omega(k)}$. Assume this pattern repeats up to $C\log k$ times, for a sufficiently large constant $C>0$. Then the probability that there is an $i$-fixing round with $V'(t_r,i)=0$ is at least $1-e^{-\Omega(k)}$. It remains to estimate the number of rounds spent in the state $V'(t_{r},i)=0$. Since each round has probability at least $1/(3n)$ to be an $i$-round, among any $r_{\text{fix}}\coloneq 6Cn\log k$ rounds there will be at least $C\log k$ $i$-rounds with probability at least $1-e^{-\Omega(k)}$. In particular, if we spend $6Cn\log k$ rounds in the state $V'(t_{r},i)=0$, then with probability at least $1-e^{-\Omega(k)}$ at least $C\log k$ of them will be $i$-rounds. By a union bound, the probability that there is an $i$-fixing round with $V'(t_r,i)=0$ within $r_{\text{total}} \coloneq r_{\text{init}} + C \log k  r_{\text{return}} + r_{\text{fix}}$ rounds is $1-\BigO(e^{-\Omega(k)}\log k) \geq 1-e^{-\Omega(k)}$, where the latter bound holds if $k$ is sufficiently large.

By a union bound over all $i$, with probability $1-ne^{-\Omega(k)}$ all indices will be fixed after at most $r_{\text{total}} \in \BigO(\Tmax k \log k)$ steps. Choosing $k = f^{1/3}\log \Tmin/(\log f(n) + \log\log \Tmin)$ gives $ne^{-\Omega(k)} \leq \exp\large\{-f(n)^{1/4}\large\}$ and $r_{\text{total}} \leq r_0$, both with room to spare. This proves the lemma.
\end{proof}


Finally we are ready to prove Theorem~\ref{thm:majority}.

\begin{proof}[Proof of Theorem~\ref{thm:majority}]
The theorem essentially follows from Lemma~\ref{lem:MAwhp} by using restarts. Let $f(n) \in \omega(1)$ be a growing function such that $f(n) \leq n$. We define a sequence $(T_i)_{i\geq 0}$ recursively by $T_0 \coloneq \Tmin = \max\{\tinit, n\log^2 n\}$ and $T_{i+1} \coloneq \sqrt{f(n)} T_i$. Moreover, we define $r_i \coloneq \eps f(n) T_i \log T_i$, where $\eps >0$ is the constant from Lemma~\ref{lem:MAwhp}. Note that $T_i$ and $r_i$ are chosen such that when we start with any \gptree of size $T_i$, then with probability at least $1-\exp\large\{-f(n)^{1/4}\large\}$ a global optimum is found within the next $r_{i+1}$ rounds without exceeding size $T_{i+1}$.

By Lemma~\ref{lem:MAwhp} there is a high chance to find an optimum in $r_0$ rounds without increasing the size of the \gptree too much. In this case, the optimization time is at most $r_0$. For the other case, the probability that either the global optimum is not found or the size of the \gptree exceeds $T_1$ is at most $p\coloneq \exp\large\{-f(n)^{1/4}\large\}$. Let $t_1$ be the \gptree at the first point in time where something goes wrong. I.e., we set $t_1$ to be the first \gptree of size larger than $T_1$, if this happens within the first $r_0$ rounds; otherwise we set $t_1$ to be the \gptree after $r_0$ rounds. In either case, $t_1$ is a \gptree of size at most $T_1$. Then we do a restart, i.e., we apply Lemma~\ref{lem:MAwhp} again with $t_1$ as the starting tree. Similar as before, there is a high chance to find an optimum in $r_1$ rounds without blowing up the \gptree too much. Otherwise (with probability at most $p$), we define $t_2$ to be the first \gptree with size at least $T_2$, if such a tree exists before round $r_0+r_1$; otherwise, we let $t_2$ be the tree at time $r_0+r_1$. Repeating this argument, the expected optimization time $T_{\text{opt}}$ is at most
\begin{align*}
\EE[T_{\text{opt}}] & \leq r_0 + p\left(r_1 + p\left(r_2 +p\left(\ldots \right)\right) \right) 
 = \sum_{i=0}^\infty p^i r_i = \eps f(n) \sum_{i=0}^\infty p^i T_i \log T_i
\end{align*}
By the recursive definition we see that $T_i = f(n)^{i/2} \Tmin$. In particular, using that $p\sqrt{f(n)} < 1/2$ for sufficiently large $n$ we obtain
\begin{align*}
\EE[T_{\text{opt}}] & \leq \eps f(n) \sum_{i=0}^\infty 2^{-i} \Tmin \log \left(f(n)^{i/2}\Tmin\right) \\
& = \eps f(n)\Tmin  \left(\log(\Tmin)\sum_{i=0}^\infty 2^{-i} + \log\left(f(n)\right)\sum_{i=0}^\infty 2^{-i}\frac{i}{2} \right)\\
& \stackrel{f(n) < n < \Tmin}{\leq} 3 \eps f(n) \Tmin \log \Tmin.
\end{align*}
This shows that for every arbitrarily slowly growing function $f(n)$ we have $\EE[T_{\text{opt}}] \leq 3 \eps f(n) \Tmin \log \Tmin$. We claim that we may replace the function $f(n)$ by a constant, i.e., that $\EE[T_{\text{opt}}]  \leq 3\eps C \Tmin \log \Tmin$ for a suitable constant $C>0$. Assume otherwise for the sake of contradiction, i.e., assume that for every constant $C>0$ there are arbitrarily large $n_C$ and \gptrees $t_C$ of size $T_C$ such that $\EE[T_{\text{opt}} \mid t_{\text{init}} = t_C] > 3\eps C T_C \log T_C$. Then we choose a growing sequence $C_i$ (for instance $C_i=i$). Since for each $C_i$ there are arbitrarily large counterexamples $n_{C_i}, t_{C_i}$, we may choose a growing sequence $n_{C_1} <n_{C_2}<n_{C_3}< \ldots$ of counterexamples. Now we define $f(n) \coloneq \min\{i \mid n_{C_i} > n\} \in \omega(1)$ and obtain a contradiction, since we have an infinite sequence of counterexamples for which $\EE[T_{\text{opt}}] > 3 \eps f(n) \Tmin \log \Tmin$. Hence we have shown for a suitable constant $C>0$ that $\EE[T_{\text{opt}}]  \leq 3\eps C \Tmin \log \Tmin$. This proves the theorem, since $\Tmin \log\Tmin \in \Theta(\max\{\tinit \log\tinit, n\log^3 n\})$.
\end{proof}

\section{Conclusion}
\label{sec:conclusion}
We considered a simple mutational genetic programming algorithm, the \oneonegp, and studied the two simple problems \order and \majority. It turns out that for these optimization is efficient in spite of the possibility of bloat: except for logarithmic factors, all run times are linear. However, bloat and the variable length representations were not easily analyzed, but required rather deep insights into the optimization process and the growth of the \gptrees.

For optimization preferring smaller \gptrees we observed a very efficient optimization behavior: whenever there is a significant number of redundant leaves, these leaves are being pruned. Whenever only few redundant leaves are present, the algorithm easily increases the fitness of the \gptree.

For optimization without consideration of the size of the \gptrees, we were able to show that the extent of bloat is not too excessive during the optimization process, meaning that the tree is only larger by multiplicative polylogarithmic factors. While such factors are not a major obstacle for a theoretical analysis, a solution which is not even linear in the optimal solution might not be desirable from a practical point of view. For actually obtaining small solutions, some kind bloat control should be used.

From our analysis we witnessed an interesting option for bloat control: by changing the probabilities such that deletions are more likely than insertions we would observe in the presented drift equations a bias towards shorter solutions. Overall, this would lead to faster optimization.

\bibliographystyle{plain}


\begin{thebibliography}{10}

\bibitem{CormenAlgorithms}
T.~H. Cormen, C.~E. Leiserson, R.~L. Rivest, and C.~Stein.
\newblock {\em Introduction to Algorithms}.
\newblock MIT~Press, 2. edition, 2001.

\bibitem{doerr2013adaptive}
Benjamin Doerr and Leslie~Ann Goldberg.
\newblock Adaptive drift analysis.
\newblock {\em Algorithmica}, 65(1):224--250, 2013.

\bibitem{dubhashi2009concentration}
Devdatt~P Dubhashi and Alessandro Panconesi.
\newblock {\em Concentration of measure for the analysis of randomized
  algorithms}.
\newblock Cambridge University Press, 2009.

\bibitem{GPFOGA11}
Greg Durrett, Frank Neumann, and Una-May O'Reilly.
\newblock Computational complexity analysis of simple genetic programming on
  two problems modeling isolated program semantics.
\newblock In {\em Proc.~of FOGA'11}, pages 69--80, 2011.

\bibitem{GoldbergO98}
David~E. Goldberg and Una-May O'Reilly.
\newblock Where does the good stuff go, and why? {H}ow contextual semantics
  influences program structure in simple genetic programming.
\newblock In {\em Proc.~of EuroGP'98}, pages 16--36, 1998.

\bibitem{grimmett2001probability}
Geoffrey Grimmett and David Stirzaker.
\newblock {\em Probability and random processes}.
\newblock Oxford University Press, 2001.

\bibitem{HeYao:04:drift}
Jun He and Xin Yao.
\newblock A study of drift analysis for estimating computation time of
  evolutionary algorithms.
\newblock {\em Natural Computing}, 3(1):21--35, 2004.

\bibitem{Joh:th:10}
Daniel Johannsen.
\newblock {\em Random Combinatorial Structures and Randomized Search
  Heuristics}.
\newblock PhD thesis, Universit{\"a}t des Saarlandes, 2010.

\bibitem{kotzing2016concentration}
Timo K{\"o}tzing.
\newblock Concentration of first hitting times under additive drift.
\newblock {\em Algorithmica}, 75(3):490--506, 2016.

\bibitem{DBLP:conf/gecco/KotzingNS11}
Timo K{\"o}tzing, Frank Neumann, and Reto Sp{\"o}hel.
\newblock {PAC} learning and genetic programming.
\newblock In {\em Proc.~of GECCO'11}, pages 2091--2096, 2011.

\bibitem{KoeSutNeuOre:c:12}
Timo K{\"o}tzing, Andrew~M. Sutton, Frank Neumann, and Una-May O'Reilly.
\newblock The {M}ax problem revisited: {t}he importance of mutation in genetic
  programming.
\newblock In {\em Proc.~of GECCO'12}, pages 1333--1340, 2012.

\bibitem{koetzing2018crossover}
Timo Kötzing, J.~A.~Gregor Lagodzinski, Johannes Lengler, and Anna
  Melnichenko.
\newblock Destructiveness of lexicographic parsimony pressure and alleviation
  by a concatenation crossover in genetic programming.
\newblock {\em CoRR}, abs/1805.10169, 2018.
\newblock (to appear in \emph{Proc.~of~PPSN'18}).

\bibitem{lengler2018drift}
Johannes Lengler and Angelika Steger.
\newblock Drift analysis and evolutionary algorithms revisited.
\newblock {\em {Combinatorics, Probability \& Computing}}, 2018.
\newblock (to appear).

\bibitem{Luke:2002:GECCO}
Sean Luke and Liviu Panait.
\newblock Lexicographic parsimony pressure.
\newblock In {\em Proc. of {GECCO}'02}, pages 829--836, 2002.

\bibitem{MambriniM14}
Andrea Mambrini and Luca Manzoni.
\newblock A comparison between geometric semantic {GP} and cartesian {GP} for
  {B}oolean functions learning.
\newblock In {\em Proc.~of GECCO'14}, pages 143--144, 2014.

\bibitem{MambriniO16}
Andrea Mambrini and Pietro~Simone Oliveto.
\newblock On the analysis of simple genetic programming for evolving {B}oolean
  functions.
\newblock In {\em Proc.~of EuroGP'16}, pages 99--114, 2016.

\bibitem{BookMitUp}
Michael Mitzenmacher and Eli Upfal.
\newblock {\em Probability and Computing: Randomized Algorithms and
  Probabilistic Analysis}.
\newblock Cambridge University Press, New York, NY, USA, 2005.

\bibitem{MoraglioMM13}
Alberto Moraglio, Andrea Mambrini, and Luca Manzoni.
\newblock Runtime analysis of mutation-based geometric semantic genetic
  programming on {B}oolean functions.
\newblock In {\em Proc.~of FOGA'13}, pages 119--132, 2013.

\bibitem{NeuGECCO12}
Frank Neumann.
\newblock Computational complexity analysis of multi-objective genetic
  programming.
\newblock In {\em Proc.~of GECCO'12}, pages 799--806, 2012.

\bibitem{NguUrlWag:c:13:Foga}
Anh Nguyen, Tommaso Urli, and Markus Wagner.
\newblock Single- and multi-objective genetic programming: new bounds for
  weighted {ORDER} and {MAJORITY}.
\newblock In {\em Proc.~of FOGA'13}, pages 161--172, 2013.

\bibitem{OReilly:thesis}
Una-May O'Reilly.
\newblock {\em An Analysis of Genetic Programming}.
\newblock PhD thesis, Carleton University, Ottawa, Canada, 1995.

\bibitem{OReilly:1994:GPSAHC}
Una-May O'Reilly and Franz Oppacher.
\newblock Program search with a hierarchical variable length representation:
  Genetic programming, simulated annealing and hill climbing.
\newblock In {\em Proc. of PPSN'94}, pages 397--406, 1994.

\bibitem{witt:2013:cpc}
Carsten Witt.
\newblock Tight bounds on the optimization time of a randomized search
  heuristic on linear functions.
\newblock {\em Combinatorics, Probability and Computing}, 22(2):294–318,
  2013.

\end{thebibliography}

\end{document}